\documentclass[letterpaper]{article}
\usepackage{uai2019}
\usepackage[margin=1in]{geometry}
\usepackage{times}

\usepackage{xcolor}
\usepackage{microtype}
\usepackage{graphicx}
\usepackage{subfigure}
\usepackage{booktabs} 
\usepackage{hyperref}
\usepackage{amsmath}
\usepackage{amsthm}
\usepackage{amssymb} 
\usepackage{mathrsfs}
\usepackage{latexsym}
\usepackage{amsmath,amsfonts,amssymb}
\usepackage{eurosym}
\usepackage{bbm}
\usepackage{colortbl}
\usepackage{algorithm}
\usepackage{algorithmic}

\definecolor{LightGray}{RGB}{230,230,230}

\usepackage{tikz}
\usetikzlibrary{fit,positioning,arrows,calc}
\tikzset{
  main/.style={circle, minimum size = 5mm, thick, draw =black!80, node distance = 10mm},
  connect/.style={-latex, thick},
  box/.style={rectangle, draw=black!100},
  edge/.style={->,> = latex'}
}

\DeclareMathOperator*{\diag}{diag}
\DeclareMathOperator{\Tr}{Tr}
\DeclareMathOperator{\Cond}{Cond}
\DeclareMathOperator{\vect}{vec}

\newcommand{\EX}[1]{\mathbb{E}_{X\sim P_{\theta}}\left[#1\right]}
\newcommand{\E}{\mathbb{E}}
\newcommand{\R}{\mathbb{R}}
\newcommand{\B}{\mathcal{B}}
\newcommand{\X}{\mathbb{X}}
\newcommand{\FM}{G}
\newcommand{\FIM}{\FM_{\theta}}

\newcommand{\trans}{\mathrm{T}}
\newcommand{\norm}[1]{\lVert #1 \rVert}
\newcommand{\abs}[1]{\lvert #1 \rvert}

\newcommand{\asto}{\stackrel{\mathrm{a.s.}}{\rightarrow}}
\newcommand{\lln}{LLN}

\newcommand{\numin}{\nu_{\mathrm{min}}}
\newcommand{\nucmin}{\nu_{\Sigma,\mathrm{min}}}
\newcommand{\nummin}{\nu_{\mu,\mathrm{min}}}

\newcommand{\nuclim}{\nu_{\Sigma,\mathrm{lim}}}

\newcommand{\hnu}[1]{\Lfh}
\newcommand{\tnabla}{\tilde \nabla}
\newcommand{\HMat}{H}

\renewcommand{\geq}{\geqslant}
\renewcommand{\leq}{\leqslant}

\renewcommand{\epsilon}{\varepsilon}
\renewcommand{\phi}{\varphi}

\providecommand{\mea}{\mu}
\providecommand{\leb}{\mu_\mathrm{Leb}}
\providecommand{\borel}{\mathcal{F}}
\providecommand{\Lf}{L_{f}}
\providecommand{\Uf}{U_{f}}
\providecommand{\dd}{\mathrm{d}}
\providecommand{\Lfh}{\widehat{\Lf}}

\newtheorem{lemma}{Lemma}[section]
\newtheorem{proposition}{Proposition}[section]
\newtheorem{remark}{Remark}[section]

\title{NGO-GM: Natural Gradient Optimization for Graphical Models}

\author{ {\bf Eric Benhamou} \\
A.I Square Connect\thanks{\,\, A.I Square Connect, Neuilly sur Seine, France}, Lamsade\thanks{\,\, Lamsade, Université Paris Dauphine, PSL, France}\\
\And
{\bf Jamal Atifl}  \\
Lamsade$^{\dagger}$\\
\And
{\bf Rida Laraki}   \\
Lamsade$^{\dagger}$\\
\And
{\bf David Saltiel}   \\
A.I Square Connect$^{*}$ \\
}

\begin{document}

\maketitle

\begin{abstract}
This paper deals with estimating model parameters in graphical models. We reformulate it as an information geometric optimization problem and introduce a natural gradient descent strategy that incorporates additional meta parameters. We show that our approach is a strong alternative to the celebrated EM approach for learning in graphical models. Actually, our natural gradient based strategy leads to learning optimal parameters for the final objective function without artificially trying to fit a distribution that may not correspond to the real one. We support our theoretical findings with the question of trend detection in financial markets and show that the learned model performs better than traditional practitioner methods and is less prone to overfitting.
\end{abstract}


\section{Introduction}
One of the most challenging question in social sciences and in particular financial markets, is to be able from past observations to make some meaningful predictions. Part of the challenge comes from multiple difficulties. First of all, there is no universally established physical law and our model will be at best a simplified version and at worst a complete non sense. Second, sequential information does not imply stationarity of the process: data may incur regime changes. Third, we should find the right balance between not enough and too much modeling. 

To be able to represent connection between our variables, an efficient and meaningful framework is graphical models. As stated in \cite{Jordan_2012}, \textit{graphical models are a marriage between probability theory and graph theory}. They are powerful as they provide a condensed way to represent variables dependencies. The graphical representation allows not only compacting information, it also provides a powerful formalism for reasoning under uncertainty. It represents knowledge about the dynamics of the variables, their dependencies and their conditional distribution, hence sometimes called also Dynamic Bayesian Networks (DBN). 

Graphical models exploit latent variables to make the model better in terms of explanation power. By defining a joint distribution over visible and latent variables, the corresponding distribution of the observed variables is obtained by marginalization. This has the nice property to express relatively complex distributions in terms of more tractable joint distributions over the expanded
variable space. The canonical atomic example of hidden variable models is the mixture distribution in which the hidden variable is the discrete component label that provides the corresponding distribution for the observable variable. The static version leads to the Gaussian mixture model and in continuous space to the factor analysis model while the dynamic version leads respectively to HMM and the Kalman filter model, often referred to as the state space model and represented by figure~\ref{DBN1}. However, dynamic graphical models can be more complex as given by figure~\ref{DBN2} with combination of the Kalman filter (KF) model and echo state networks (ESN) or by figure~\ref{DBN3} with feedback from past observations to next step latent variables. We will exploit a combination of the latter in our numerical experiments.

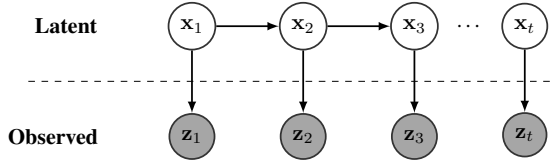
\begin{figure}[ht]
\centering
\resizebox {.45\textwidth} {!} {
\begin{tikzpicture}
 \node[box,draw=white!100] (Latent) {\textbf{Latent}};
 \node[main] (L1) [right=of Latent] {$\mathbf{x}_1$};
 \node[main] (L2) [right=of L1] {$\mathbf{x}_2$};
 \node[main] (L3) [right=of L2] {$\mathbf{x}_3$};
 \node[main] (Lt) [right=of L3] {$\mathbf{x}_t$};
 \node[main,fill=black!35] (O1) [below=of L1] {$\mathbf{z}_1$};
 \node[main,fill=black!35] (O2) [below=of L2] {$\mathbf{z}_2$};
 \node[main,fill=black!35] (O3) [below=of L3] {$\mathbf{z}_3$};
 \node[main,fill=black!35] (Ot) [below=of Lt] {$\mathbf{z}_t$};
 \node[box,draw=white!100,left=of O1] (Observed) {\textbf{Observed}};

 \path (L1) edge [connect] (L2)
        (L2) edge [connect] (L3)
        (L3) -- node[auto=false]{\ldots} (Lt);

 \path (L1) edge [connect] (O1)
	(L2) edge [connect] (O2)
	(L3) edge [connect] (O3)
	(Lt) edge [connect] (Ot);

 \draw [dashed, shorten >=-0.5cm, shorten <=-0.5cm]
      ($(Latent)!0.5!(Observed)$) coordinate (a) -- ($(Lt)!(a)!(Ot)$);
\end{tikzpicture}
}
\caption{State Space model represented as a Bayesian Probabilistic Graphical model. Each vertical slice represents a time step. Nodes in gray (resp. white) represent observable (resp. non observable or latent) variables. This State Space model encompasses HMM and KF models} \label{DBN1}
\end{figure}

\begin{figure}[ht]
\centering
\resizebox {.47\textwidth} {!} {
\begin{tikzpicture}
 \node[box,draw=white!100] (Latent) {\textbf{Latent}};
 \node[main] (L1) [right=of Latent] {$\mathbf{x}_1$};
 \node[main] (L2) [right=of L1] {$\mathbf{x}_2$};
 \node[box,draw=white!100] (L4) [right=of L2] {};
 \node[main] (Lt) [right=of L4] {$\mathbf{x}_t$};
 \node[main] (L21) [below=0.5cm of L1 ] {$\mathbf{y}_1$};
 \node[main] (L22) [below=0.5cm of L2] {$\mathbf{y}_2$};
 \node[box,draw=white!100] (L24) [below=1.01cm of L4] {};
 \node[main] (L2t) [below=0.5cm of Lt] {$\mathbf{y}_t$};
 \node[main,fill=black!35] (O1) [below=of L21] {$\mathbf{z}_1$};
 \node[main,fill=black!35] (O2) [below=of L22] {$\mathbf{z}_2$};
 \node[box,draw=white!100] (O4) [below=of L24] {};
 \node[main,fill=black!35] (Ot) [below=of L2t] {$\mathbf{z}_t$};
 \node[box,draw=white!100,below=80pt] (Observed) {\textbf{Observed}};
 \path (L1) edge [connect] (L2)
        (L2) -- node{\ldots} (L4)
        (L4) edge [connect] (Lt);
 \path (L21) edge [connect] (L22)
        (L22) -- node{\ldots} (L24)
        (L24) edge [connect] (L2t);
 \path (L21) edge [connect] (O1)
	(L22) edge [connect] (O2)
	(L2t) edge [connect] (Ot);
 \path (L1) edge [bend right,connect] (O1)
	(L2) edge [bend right, connect] (O2)
	(Lt) edge [bend right, connect] (Ot);
 \path (L1) edge [connect] (L22)
        (L4) edge [connect] (L2t);
 \draw [dashed, shorten >=-0.5cm, shorten <=-2.5cm]
     ($(O1)+(0,0.852)$) -- ($(Ot)+(0,0.8)$);
\end{tikzpicture}
}\caption{Example of another dynamic Bayesian network combining Kalman filter (KF) model and echo state networks (ESN). This is another example of a multi-input several multi-outputs (MISMO) forecasting model. It is used frequently in time series forecast (see for instance \cite{Xiao_2017})} \label{DBN2}
\end{figure}
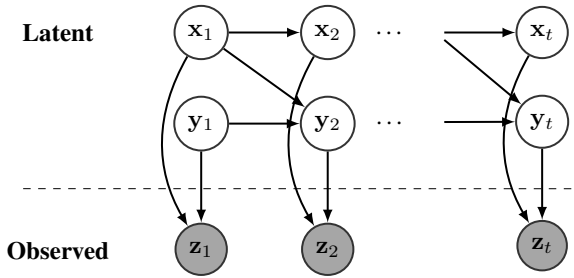

\begin{figure}[ht]
\centering
\resizebox {.47\textwidth} {!} {
\begin{tikzpicture}
 \node[box,draw=white!100] (Latent) {\textbf{Latent}};
 \node[main] (L1) [right=of Latent] {$\mathbf{x}_1$};
 \node[main] (L2) [right=of L1] {$\mathbf{x}_2$};
 \node[box,draw=white!100] (L4) [right=of L2] {};
 \node[main] (Lt) [right=of L4] {$\mathbf{x}_t$};
 \node[main] (L21) [below=0.5cm of L1 ] {$\mathbf{y}_1$};
 \node[main] (L22) [below=0.5cm of L2] {$\mathbf{y}_2$};
 \node[box,draw=white!100] (L24) [below=1.01cm of L4] {};
 \node[box,draw=white!100] (L24bis) [below=1.2cm of L4] {};
 \node[main] (L2t) [below=0.5cm of Lt] {$\mathbf{y}_t$};
 \node[main,fill=black!35] (O1) [below=of L21] {$\mathbf{z}_1$};
 \node[main,fill=black!35] (O2) [below=of L22] {$\mathbf{z}_2$};
 \node[box,draw=white!100] (O4) [below=of L24] {};
 \node[box,draw=white!100] (O4bis) [below=1.5cm of L24] {};
 \node[main,fill=black!35] (Ot) [below=of L2t] {$\mathbf{z}_t$};
 \node[box,draw=white!100,below=90pt] (Observed) {\textbf{Observed}};
 \path (L1) edge [connect] (L2)
        (L2) -- node{\ldots} (L4)
        (L4) edge [connect] (Lt);
 \path (L21) edge [connect] (L22)
        (L22) -- node{\ldots} (L24)
        (L24) edge [connect] (L2t);
 \path (L21) edge [connect] (O1)
	(L22) edge [connect] (O2)
	(L2t) edge [connect] (Ot);
 \path (L1) edge [bend right,connect] (O1)
	(L2) edge [bend right, connect] (O2)
	(Lt) edge [bend right, connect] (Ot);
 \path (O1) edge [connect] (L22)
	(O4bis) edge [connect] (L2t);

 \draw [dashed, shorten >=-0.5cm, shorten <=-2.5cm]
     ($(O1)+(0,0.855)$) -- ($(Ot)+(0,0.8)$);
\end{tikzpicture}
}\caption{Example of a dynamic Bayesian network with connection between past observable variables and latent variables} \label{DBN3}
\end{figure}
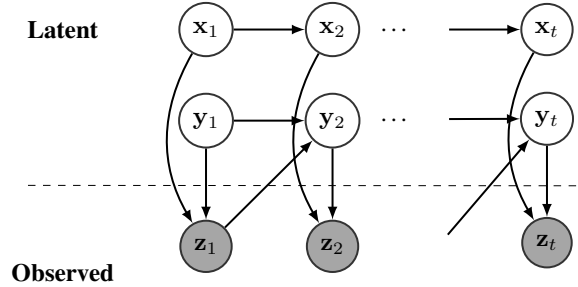
To finalize our model, we need to solve the issue of learning the its parameters. The typical learning approach is the Expectation Maximization (EM) algorithm. It was initially developed for mixture models, in particular Gaussian mixtures and other natural laws such as Poisson, binomial, multinomial and exponential distributions, see \cite{Hartley_1958} and \cite{Dempster_1977}. It was only when the link between latent variables and Kalman filter models was made that it became clear that it could also be applied to Kalman and extended Kalman filters (see \cite{Cappe_2010} or \cite{Einicke_2010}). The EM method is so far the state of the art method for learning graphical model parameters as it provides an efficient way to find model parameters in a fraction of seconds (see for instance \cite{Neal_1999}, but also  \cite{Pfeifer_2018}, \cite{Li_2017}, \cite{Robin_2017}, \cite{Levine_2018}). Interestingly \cite{XuJordan_1996} shows that the EM method can be viewed as a gradient descent where the decrement is computed as the projection of the gradient, making it a variable metric gradient ascent. Similarly, \cite{Amari_1995} proves that the E and M steps in EM can be interpreted as dual projected gradient flows under dual affine connections using information geometry. This advocates to find alternatives that are also gradient descent but in the natural space.

We argue that although EM enjoys several nice properties, it misses the point that graphical models are imprecise and simplified models for the reality especially when tackling complex problems such as time series forecasting. In particular, whenever we apply graphical models to economics and finance, we are forced to make some modeling assumptions about the state dynamics and the graph topology (the DBN structure). These assumptions may be incorrectly specified and add some bias compared to reality. Trying to use a best fit approach through maximum likelihood estimation and Kullback Leibler divergence optimization may miss this point and try to fit at all cost the model on data. It does not factor in the interdependence between our graphical model and the actions related to this graphical model. In the case of social sciences, if the graphical model is used to make a forecast which is then used to make a parametric action, the EM method does not take into account the interrelation between parameters of the graphical model denoted by $\theta$ and the ones of the action denoted by $\tau$. To measure the efficiency of the full set of parameters $(\theta, \tau)$, we rely on a cost function that can contain a regularizer to lead to smooth parameters. In finance, this cost function can be for instance a measure of the performance of our prediction taking into account the risk in our actions.

We present here a new approach that takes a radical point of view and focuses on the final efficiency of our model. Graphical model parameters are now estimated in terms of their efficiency for the cost function together with the action parameters rather than just their distributional fit to data. We rely on information geometric optimization to find a local optimum for our final cost function.
Our key features are the following:
{
\begin{itemize}
\setlength\itemsep{0em}
\item it is possible to directly optimize the cost function with a stochastic optimization approach and in particular the CMA-ES method;
\item this approach computes a natural gradient in the implicit Fisher information Riemannian manifold;
\item it is a good alternative to the EM approach as it does not fit at any cost the distribution of our graphical model to reality but rather looks at model efficiency measured by a loss function related to the problem under study;
\item the estimation of our model parameters is performed jointly with the action strategy parameters;
\item numerical results show that the overfitting issue of this approach due to local minima is less than the EM approach as it takes into account that the model dynamics is incorrectly specified.
\end{itemize}
}
The rest of the paper is organized as follows. Section~\ref{Settings} presents the overall framework and the resulting optimization problem. Section~\ref{IGO} provides some theoretical arguments that favor stochastic optimization approaches based on Information Geometric Optimization (IGO). Section~\ref{NumericalResults} discusses an example in finance. Our method outperforms traditional trend following methods by far. We finally conclude about some possible extensions and further experiments.

\section{Settings} \label{Settings}
Suppose we have determined an architecture for our Dynamic Bayesian Network. This may be inspired by combinations of simple network architectures such as those in Figures~\ref{DBN1},~\ref{DBN2},~\ref{DBN3}. This model is used for some specific goal. In our case, it is used to forecast some times series. But this is not our final objective! We are interested in using this forecast to perform a specific action. In the case of a financial market algorithmic trading strategy, we will use the forecast to make an informed decision and decide whether we should buy, sell or do nothing with a financial asset. To keep things simple, we will assume that when we take our decision, we target a pre-determined movement amplitude, materialized by a profit target and a stop loss level. The signal for the action is given by the difference between the forecast generated by our graphical model and the last price. To avoid false signals, we add an additional threshold parameter for our action and consider that there is an uptrend signal (respectively a downtrend signal) if the forecast is above the last price plus a threshold (respectively below the last price minus a threshold). Using a fixed price target, a stop loss and a threshold is quite realistic and is done by many practitioners as presented in various papers, e.g. \cite{Mauricio_2010}, \cite{Graziano_2014}, \cite{Stanley_2017}, or \cite{Vezeris_2018}. The profit target ensures that the strategy locks the profit in real money and is technically corresponding to a limit order while the stop loss, technically corresponding to a stop order, safeguards the overall risk by limiting losses whenever the market backfires and contradicts the presumed direction.

The price target, stop loss and threshold are three additional parameters that govern our action. These parameters are closely interconnected to our network parameters. If the network's prediction is accurate enough, we should aim for a large price movement, hence a large price target and a small stop loss. If it is not accurate enough, on the contrary, we should reduce the price target and increase the stop loss. We hold the position until either the trade reaches the profit target and exits in a profit or it touches the stop loss level and exits with a loss. The objective function to measure the performance of the model is the eponymous Sharpe ratio (introduced in \cite{Sharpe_1966}) that corresponds to the ratio of the strategy net return over its standard deviation. We add an $L1$ regularization term to ensure sparsity of the dynamic graphical model parameters. Although this approach may recall reinforcement learning theory, the dynamic nature of our graphical model makes the usage of standard reinforcement learning tools such as Q-learning inappropriate.

On this simple example, we clearly see that optimizing graphical model parameters in a standalone fashion is not optimal and that these parameters should be evaluated together with the action parameters at the light of our loss function as represented by figure~\ref{challenge}. We want to maximize the Sharpe ratio of our trading strategy over time. This optimization problem is non convex with potentially many local optima. Moreover, the binary nature of our action, i.e. we sell or buy as soon as our forecast hits the past level plus or minus the threshold, leads to discontinuities. 

The fundamental difference between EM method and ours is to jointly optimize the model and the action parameters, learning $(\theta, \tau)$. Because model and action parameters are interconnected, this approach well fits our initial desiderata. It optimizes our final criteria, the loss function, and not the maximum likelihood of our graphical model.  We will present in the next section some important results from information geometric optimization that allow us to tackle this complex optimization problem using a natural gradient descent.

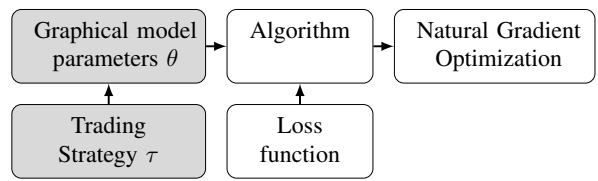
\begin{figure}[ht]
\centering
\resizebox {0.48\textwidth} {!} {
  \begin{tikzpicture}[every text node part/.style={align=center}]
    \node (1) [draw, rectangle,rounded corners,text width=2.6cm, fill=black!15] {\begin{tabular}{c} Graphical model \\  parameters $\theta$ \end{tabular}};
    \node (2) [below=0.3cm of 1, draw, rectangle,rounded corners,text width=2.6cm, fill=black!15] {\begin{tabular}{c}Trading \\  Strategy $\tau$ \\ \end{tabular}};

    \node (3) [right=0.3cm of 1, draw, rectangle, rounded corners, text width=1.9cm] { \begin{tabular}{c} Algorithm \\ \\ \end{tabular}};
    \node (4) [below=0.3cm of 3, draw, rectangle, rounded corners, text width=1.9cm] {\begin{tabular}{c}Loss \\ function  \\  \end{tabular}};

    \node (5) [right=0.3cm of 3, draw, rectangle, rounded corners, text width=2.7cm]  {\begin{tabular}{c}Natural Gradient\\ Optimization\\ \end{tabular}};

%

\path (2) edge [connect] (1)
	 (4) edge [connect] (3);

\path (1) edge [connect] (3)
	 (3) edge [connect] (5);

  \end{tikzpicture}
}\caption{Learning process for our graphical model. Decide an architecture for our graphical model with parameters $\theta$, combine with a strategy with parameters $\tau$ to create a systematic algorithm. Select a loss function. Use natural gradient optimization to find best parameters $( \theta^{\star}, \tau^{\star})$ }\label{challenge}
\end{figure}

\section{Information Geometric Optimization}\label{IGO}
\subsection{Natural gradient descent}
Compared to the usual gradient descent method, the natural gradient takes into account the intrinsic geometric structure of the underlying Riemmanian space. Since the seminal work of \cite{Amari_1998}, it has been widely adopted by the machine learning community \cite{Pascanu_2014}, \cite{Bernacchia_2018}.  In gradient descent, the usual update step is $\theta^{t+1} = \theta^{t} - \nu_t \frac{\partial U(\theta)}{\partial \theta^t}$ where $\theta^{t}$ is the optimization parameter at iteration $t$, $\nu_t $ the learning rate, and $U$ the cost function. The method of natural gradients replacement proposes to change the gradient descent with $
\theta^{t+1} = \theta^{t} - \nu_t G^{-1}(\theta^{t}) \frac{\partial U(\theta)}{\partial \theta^t}$ where $G(\theta^{t})$ is a matrix that defines the Riemannian metric in the space of the parameters. This is very powerful as it computes the fastest descent when looking at the Riemannian metric induced by the parameters. In distributions space, the Riemannian metric is associated with the Fisher information matrix \cite{Rao_1945,Jeffreys_1946}. In general, it is difficult to apply the natural gradient descent because we need to invert the Fisher information matrix, which is computationally heavy, except in some particular cases where there exists a closed form like for instance for exponential family distributions. We will follow the seminal works of \cite{Ollivier_2017} and \cite{Akimoto_2012b} and present a stochastic optimization method that performs numerically natural gradients efficiently.

We denote by $\X$ a metric space and associate to $\X$ a Borel $\sigma$-field and a measure on $\X$ denoted by $\borel$ and $\mea$ respectively. Typically $\X$ is $\R^d$ and $\mea$ is the Lebesgue measure on $\R^d$. We are interested in minimizing a $\mea$-measurable real value function $f: \X \to \R$. In order to make our optimization invariant with respect to various standard transformations of $f$, instead of minimizing $f$, it is better to find the minimum of a \emph{loss} (also referred to as an \emph{invariant cost}) function that is invariant to any strictly increasing transformation of $f$. We define $\Lf$ as the $\mea$-measured volume of the unit ball in the functional space $f$ with radius $f(x)$, that is 
$$
\Lf: x \mapsto \mea[y: f(y) \leq f(x)].
$$
Said differently, $\Lf$ is the measure of all the elements whose value is less or equal to $f(x)$. We are interested in finding the optimum of the loss function with respect to a family of probability distributions $P_{\theta}$ on $\X$, $\theta \in \Theta$. Compared to standard optimization, we define a \textit{quasi}-objective function, sometimes referred to as a utility function, $\Uf$, on the parameter space $\Theta$ defined as the expected value of our loss function $\Lf$ over the space of distributions to the power $2/d$: $P_{\theta}$, namely $ \Uf(\theta) = \EX{\Lf^{2/d}(X)}$. The intuition behind the exponent $2/d$ is to ensure that the utility is in a sense dimensionless and homogeneous to the square of a Euclidean distance. This is because in dimension $d$, a unit ball is homogeneous to a radius of dimension $d$, where as the square of a Euclidean distance is of dimension $2$. We will see in the practical example of a a quadratic function (see subsection \ref{sec:conv})  that this choice of exponent makes the explicit computation of the natural gradient easy.

In \cite{Ollivier_2017}, this utility function is defined as the opposite of the weighted quantile, $\Uf^{IGO}(x) = - w(P_{\theta^{t}}[y: f(y) \leq f(x)])$, where $w: [0, 1] \mapsto \R$ is non-increasing weight function. This choice is not easy to analyze as the quantile $P_{\theta^{t}}[y: f(y) \leq f(x)]$ depends on the current parameter $\theta^{t}$. At each iteration, the loss function $\Lf^{IGO}(x)$ changes making it hard to analyze. In \cite{Akimoto_2012b}, the utility (referred to as the quasi objective) function is defined as $\Uf^{NGA}=\EX{\nu[y: f(y) \leq f(x)]}$ where $\nu$ represents any monotonically increasing set function on $\borel$, i.e., $\nu(A) \leq \nu(B)$ for any $A$, $B \in \borel$ s.t.\ $A \subseteq B$. Our function is somehow simpler as the loss function is just the expected volume of all the elements whose value is less or equal to $f(X)$ for $X \sim P_{\theta}$ while the Utility function is similar but homogeneous to a Euclidean distanced squared.

We are interested in performing a natural gradient descent on a Riemannian manifold given by $(\Theta, G(\theta))$ equipped with the Fisher information metric $\FIM$. We choose the Fisher metric because it is the unique metric that does not depend on the choice of parameterization as explained in \cite{Amari2007book}. The natural gradient is easy to obtain and is just the gradient with respect to the Fisher information matrix metric. 
One can ``easily'' compute it and get that it is provided by the product of the inverse of the Fisher information matrix $\FIM$ with the standard or \textit{vanilla} gradient $\nabla \Uf(\theta)$ of the function to minimize. We can then apply a natural gradient descent as follows:
\begin{equation}
\theta^{t+1} = \theta^{t} - \nu_{t} \FM_{\theta^{t}}^{-1} \nabla \Uf(\theta^{t}), \label{equation:natural_gradient-f}
\end{equation}
where $\nu_{t}$ is the learning rate. Compared to the vanilla gradient, we have an additional term given by $\FM_{\theta^{t}}^{-1}$. We present below an algorithm that allows performing the natural gradient descent in a general case. To make things clearer, we will present the algorithm in the case where we can explicitly compute the inverse of the Fisher information matrix and the  standard gradient $\nabla \Uf(\theta)$. We will call this the \textit{Closed form Natural Gradient} or (NGD). We will then present the Monte-Carlo NGD that allows performing the natural gradient descent in a general framework efficiently.

Historically, this approach of computing a natural gradient in the space of distributions was done empirically and without strong theoretical arguments in the Evolution Strategies community. The most prominent approach belonging to this work which is indeed a natural gradient method is the CMA-ES algorithm \cite{HansenOstermeier_2001}. Although CMA-ES has been state of the art in this line of research as shown by the various benchmarks of the \href{http://coco.gforge.inria.fr/}{Comparing Continuous Optimisers} (COCO) INRIA platform for ill-posed functions, it was only later, after the works of \cite{Akimoto_2010}, \cite{Glasmachers_2010}, \cite{Akimoto_2012a}, \cite{Akimoto_2012b}, and the deep theoretical study of \cite{Ollivier_2017}) that the community realizes that this algorithm performs a natural gradient descent. In the following we adapt these theoretical proofs to the setting of our graphical model learning problem. The CMA-ES algorithm has led to a large number of papers and articles and we refer the interesred reader to \cite{HansenOstermeier_2001}, \cite{Auger_2004}, \cite{Igel_2007}, \cite{Auger_2009}, \cite{Hansen_2011}, \cite{Auger_2012}, \cite{Hansen_2014}, \cite{Auger_2015}, \cite{Auger_2016}, \cite{Ollivier_2017} and \cite{Hansen_2018} to cite a few.

\subsection{Closed Form (CF) NGD Algorithm}
In order to study the key property of the CMA-ES, we focus on the case where $\X = \R^{d}$ and suppose the measure $\mea$ to be the Lebesgue measure on $\R^{d}$, denoted $\leb$, while $\borel$ denotes the Borel $\sigma$-field $\B^{d}$ on $\R^{d}$. 

Concerning the sampling distribution, we will work with the Gaussian distribution as it has the maximum entropy among distributions with known first two moments. Our distribution space, denoted by $P_{\theta}$ is parameterized by the parameter $\theta \in \Theta$. We adopt the traditional moment parameters, that is the mean vector $\mu(\theta)$, which is in $\R^{d}$ and the covariance matrix $\Sigma(\theta)$, which is a symmetric and positive definite matrix of dimension $d \times d$, which leads that $P_\theta$ is indeed the normal written as $\mathcal{N}(\mu(\theta),\Sigma(\theta))$.

The loss function $\Lf(x)$ is naturally associated to the Lebesgue measure $\leb$ and defined as $\Lf(x) = \leb[y: f(y) \leq f(x)].$ 
In order to optimize our function, we look at the infimum of 
$$
\Uf(\theta) = \EX{\Lf^{2/d}(X)}.
$$ 

If there is only one optimum, this translates into finding the domain $\Theta$ where the mean vector equals the global minimum of $f$ while the covariance matrix is null. 

It is worth noticing that the choice of the moment parameterization of Gaussian distributions does affect the behavior of the natural gradient update \eqref{equation:natural_gradient-f} with finite learning rate $\nu^{t}$. However, the steepest direction of $\Uf$ on the statistical manifold $\Theta$ is invariant under the choice of parameterization as explained in \cite{Ollivier_2017}.  
Using explicit computation for the normal, we get the natural gradient descent summarized by proposition \ref{prop:natural_gradient} below.

\begin{proposition}\label{prop:natural_gradient}
If the loss function is squared integrable ($\EX{\Lf^{2}(X)} < \infty$), using two different learning rates $(\nu_{\mu}^{t}, \nu_{\Sigma}^{t} )$  for the mean vector and the covariance, the  natural gradient descent \eqref{equation:natural_gradient-f} writes:
\begin{equation}\label{equation:natural_gradient}
\begin{split}
\mu^{t+1} = \mu^{t} - \nu_{\mu}^{t} \delta \mu^{t}, \\
\Sigma^{t+1} = \Sigma^{t} - \nu_{\Sigma}^{t} \delta \Sigma^{t},
\end{split}
\end{equation}
\begin{equation*}
\begin{split}
\text{with } \, \delta \mu^{t} &= \E_{X \sim P_{\theta^{t}}}[\Lf^{2/d}(X) (X - \mu^{t})] \\
\delta \Sigma^{t} &= \E_{X \sim P_{\theta^{t}}}\bigl[ \Lf^{2/d}(X) \bigl((X - \mu^{t})(X - \mu^{t})^\trans - \Sigma^{t} \bigr) \bigr]
\end{split}
\end{equation*}
\end{proposition}

\begin{proof}
Refer to \ref{proof0} in the supplementary material for more details.
\end{proof}

Equations \eqref{equation:natural_gradient} defines the closed-form or deterministic natural gradient descent (NGD) method, which is an ideal case where we know the optimum solution. This ideal case is useful for deriving fruitful properties of our optimization method but is useless in practice.

\subsection{Monte-Carlo (MC) NGD Algorithm}\label{sec:stoc}
We now turn to the real algorithm that is efficient in real-world situations. It tackles the issue of unknown and untractable gradient for our quasi objective function: $\nabla \Uf(\theta)$. The last resort solution is to estimate the gradient with  Monte Carlo simulations. We approximate the natural gradient and simulate the natural gradient descent with the algorithm~\ref{MC-NGD}.

\begin{algorithm}
\caption{Monte-Carlo NGD Algorithm}\label{MC-NGD}
	\begin{algorithmic} 
	\WHILE{Not Converged}											
		\STATE Simulate n random normal vectors denoted by $z_i$ according $\mathcal{N}(0, Id_d)$, infer $x_i = m_t +\sqrt{\Sigma^t} z_i$, and evaluate their values $f(x_{1}), \dots, f(x_{n})$
		\STATE Estimate the Loss function $\Lf(x_{i})$ as
		\vspace{-0.2cm}
		\begin{equation*}
			\widehat{\Lf}(x_{i}) = \frac{ \sqrt{ (2 \pi)^{d} \det(\Sigma)}}{n} \!\!\!\!\!\!\!\!\! \sum_{j: f(x_{j}) \leq f(x_{i})} \exp\biggl( \frac{\norm{z_{j}}^{2}}{2} \biggr)
		\end{equation*}
		\vspace{-0.2cm}
		\STATE  Estimate the natural gradient $\delta \mu^{t}$ and $\delta \Sigma^{t}$ as
		\vspace{-0.35cm}
		\begin{equation}\label{equation:natural_gradient-est}
		\begin{split}
			\widehat{\delta \mu^{t}} &= \frac{1}{n} \sum_{i=1}^{n} \left(\widehat{\Lf}(x_{i})\right)^{2/d} (x_{i} - \mu^{t})\\
			\widehat{\delta \Sigma^{t}} &= \frac{1}{n}  \sum_{i=1}^{n} \left(\widehat{\Lf}(x_{i})\right)^{2/d}  \bigl((x_{i} - \mu^{t})(x_{i} - \mu^{t})^\trans  - \Sigma^{t}\bigr)\enspace.
		\end{split}
		\end{equation}
		\vspace{-0.35cm}
		\STATE  Update parameters with ngd as \\
			\qquad \hspace{0.55cm}  $\mu^{t+1} = \mu^{t} - \nu_{\mu} \widehat{\delta \mu^{t}}$ \\
			\qquad and $\Sigma^{t+1} = \Sigma^{t} - \nu_{\Sigma} \widehat{\delta \Sigma^{t}}$.
	\ENDWHILE
	\end{algorithmic}
\end{algorithm}

This algorithm is the Monte Carlo version of our NGD algorithm. In \cite{Ollivier_2017}, this is referred as the stochastic NGD algorithm. Compared to the closed-form NGD algorithm, we evaluate $\nabla \Uf(\theta)$ thanks to a Monte Carlo simulation. This makes this algorithm stochastic in nature. In order to have the same behavior each time we run this algorithm, we freeze random seeds to ensure similar results for the normal random numbers. 
The core of the algorithm is to generate $n$ samples $x_{i}$ from a multivariate Gaussian distribution $\mathcal{N}(\mu^{t}, \Sigma^{t})$, evaluate their value, computes the loss functions $\Lf(x_{i})$ afterwards. The estimates $\widehat{\Lf}(x_{i})$ are obtained as follows. In order to get the intuition of the Monte Carlo approximation, recall that the loss function is given by:
\begin{eqnarray*}
\Lf(x)&=&  \int{\mathbf{1}_{\{f(y) \leq f(x)\}}} \dd y ) \\
& = & \int \frac{\mathbf{1}_{\{f(y) \leq f(x)\}}}{p_{\theta^{t}}(y)} p_{\theta^{t}}(y) \dd y 
\end{eqnarray*}
In our Monte-Carlo algorithm, a first step consists in computing the integral as the corresponding Riemann sum:
\begin{equation}
\widehat{\Lf}(x) =  \frac{1}{n}\sum_{j = 1}^{n} \frac{\mathbf{1}_{\{f(x_{j}) \leq f(x)\}}}{p_{\theta^{t}}(x_{j})} \label{equation:mc-nu}
\end{equation}
For the multivariate Gaussian, the probability weights are given by $p_{\theta^{t}}(x_{j})  = 1 / \sqrt{(2\pi)^{d}\det(\Sigma)} \exp\bigl( -\norm{z_{j}}^{2}/ 2 \bigr)$ where the $z_i$ are standardized normal $\mathcal{N}(0, Id_d)$ draws. It is worth noticing that these weights are not the traditional ones of the CMA-ES algorithm. CMA-ES relies rather on logarithmic weights \cite{Hansen_2014} and takes a fraction (denoted by $\mu$, that means something different from our mean vector) out of $\lambda$ candidates. In this setting, we use all the candidates to compute the new mean and covariance. From a theoretical point of view, the weights introduced  in \cite{Akimoto_2012b} are more meaningful and make the study of the property of this stochastic optimization algorithm simpler. The next step is to estimate the natural gradient according to equations \eqref{equation:natural_gradient} as follows:
\begin{eqnarray}
& & \hspace{-0.5cm} \E_{X \sim P_{\theta^{t}}}\left[\Lf^{2/d}(X) (X - \mu^{t})\right] \nonumber\\
& & = \frac{1}{n} \sum_{i=1}^{n}  \left(\widehat{\Lf}(x_{i})\right)^{2/d} (x_{i} - \mu^{t})\\
& & \hspace{-0.5cm} \E_{X \sim P_{\theta^{t}}}\left[ \Lf^{2/d}(X) \left((X - \mu^{t})(X - \mu^{t})^\trans - \Sigma^{t} \right) \right] \nonumber\\
& & = \frac{1}{n}  \sum_{i=1}^{n}  \left(\widehat{\Lf}(x_{i})\right)^{2/d}  \left((x_{i} - \mu^{t})(x_{i} - \mu^{t})^\trans  - \Sigma^{t}\right) \nonumber \\
\end{eqnarray}

In the final natural gradient update, it is important to have both the mean and the covariance learning rates not too large. We will see in the next section the reason of the limitation of the learning rates when looking at the positivity of the covariance matrix at each step as well as when examining convergence properties.

\subsection{Basic Properties}\label{Properties}
\noindent\textbf{Invariance. }
Invariance properties are fundamental for the efficiency of optimization algorithm. Our stochastic optimization has two major invariance properties: under monotonic transformations of the objective function and under affine transformations of the search space. Invariance under monotonic transformation means the algorithm performs equally well on the function $f$ and on $g \circ f$, the composition of $f$ with any function $g$ strictly increasing. This explains why it performs well on ill-conditioned functions whereas conventional gradient methods like Newton method relies on convexity properties of the objective function and handles this non-convex problems less well. This invariance property is a direct consequence of the loss function. Invariance under affine transformations of the search space is the key idea behind the Newton's method. The adapation of the covariance matrix at each step explains the performance of this algorithm on ill-conditioned objective functions. 

\noindent\textbf{Positivity. }
The covariance matrix must be positive definite and symmetric at each iteration. Although \cite{NIPS2016_6457} recently took the additional constraint to Cholesky decompose the covariance matrix, we can explicitly compute the condition on the learning rate $\nu_{\Sigma}^{t}$ to ensure that in our NGD algorithms the covariance matrix always remains positive definite symmetric. Proposition~\ref{prop:positivity} shows that the learning rate has too be small enough. Should the learning rate be larger than the critical value, the covariance matrix would be be non positive. 

\begin{proposition}\label{prop:positivity}
Provided that the matrix $\Sigma^{0}$ and the various matrices $I_{d} - \nu_{\Sigma}^{t} \sqrt{\Sigma^{t}}^{-1} \delta \Sigma^{t} \sqrt{\Sigma^{t}}^{-1}$ are positive definite, the covariance matrix  $\Sigma^{t}$ remains positive definite (and symmetric) for each $t$ in the closed form NGD algorithm. Should the matrix $I_{d} - \nu_{\Sigma}^{t} \sqrt{\Sigma^{t}}^{-1} \delta \Sigma^{t} \sqrt{\Sigma^{t}}^{-1}$ not be positive (with eigen values negative or null), $\Sigma^{t+1}$ would not be positive. In the Monte-Carlo NGD algorithm, the condition is modified into $I_{d}-\nu_{\Sigma}^{t} \sqrt{\Sigma^{t}}^{-1} \widehat{\delta \Sigma^{t}} \sqrt{\Sigma^{t}}^{-1}$ should be positive definite for any $t$.
\end{proposition}
\begin{proof}
Refer to \ref{proof1} in the supplementary material for more details.
\end{proof}

\begin{remark}
In case, the matrix $\sqrt{\Sigma^{t}}^{-1} \delta \Sigma^{t} \sqrt{\Sigma^{t}}^{-1}$ for the Closed Form case and $\sqrt{\Sigma^{t}}^{-1} \widehat{\delta \Sigma^{t}} \sqrt{\Sigma^{t}}^{-1}$ for the Monte-Carlo case has some positive eigen values, this result can be re-expressed trivially in terms of eigenvalues and state that the learning rate is bounded by $\nu_{\Sigma}^{t} < 1/ \lambda_{1}(\sqrt{\Sigma^{t}}^{-1} \delta \Sigma^{t}   \sqrt{\Sigma^{t}}^{-1})$ for the Closed Form case and $\nu_{\Sigma}^{t} < 1/ \lambda_{1}(\sqrt{\Sigma^{t}}^{-1} \widehat{\delta \Sigma^{t}} \sqrt{\Sigma^{t}}^{-1})$ for the Monte Carlo case where $\lambda_{1}(M)$ denotes the largest eigenvalue of a matrix $M$. This is quite intuitive. The greater the critical largest eigenvalue is, the bigger the gradient is. Hence the learning rate should not be too large to go too far. Reciprocally the smaller the critical largest eigenvalue is, the smaller the gradient is. Should the matrices previously mentioned have only negative values, there would be no bound on the learning rate. This result is already presented in \cite{Akimoto_2012b} but without mentioning that the two matrices shall not be necessarily positive. We will see in the trivial case of the quadratic function in subsection \ref{sec:conv}, that the matrix in the Closed Form case can be explicitly computed and given by $\kappa \, \Sigma^{t} \HMat \Sigma^{t}$ with $\kappa$ given by proposition \ref{prop:convergence}. Should $H$ not be positive, the matrix $\sqrt{\Sigma^{t}}^{-1} \delta \Sigma^{t} \sqrt{\Sigma^{t}}^{-1}$ would not be neither.
\end{remark}

\noindent\textbf{Convergence. }
Although the gradient estimator defined in \eqref{equation:natural_gradient-est} may not be necessarily unbiased, yet it converges to the true natural gradient as proved by proposition~\ref{prop:consistency}. This implies in particular that the Monte-Carlo NGD  approximates well the closed-form NGD when we increase the sample size $n$. Let us denote by $\nabla: U \mapsto \FIM^{-1} \nabla U$ the natural gradient operator, and by $\tnabla_{{\theta}^{n}} : \widehat{U} \mapsto \FIM^{-1} \nabla \widehat{U}$ the Monte Carlo estimate of the natural gradient.

\begin{proposition}\label{prop:consistency}
Let $X_{1}, \dots, X_{n}$ be independent and identically distributed (i.i.d) random vectors following $P_{\theta}$. Let us denote $\vect(M)$ the vectorization of matrix $M$, that is the vector obtained by stacking matrix columns on top of one another. Let $\Lfh(x)$ and $\tnabla{\theta}^{n} = [(\widehat{\delta \mu^{t}})^{\trans}, \vect(\widehat{\delta \Sigma^{t}})^{\trans}]^{\trans}$ be the loss function estimator given by \eqref{equation:mc-nu} and the natural gradient estimator  where in \eqref{equation:natural_gradient-est}  the $x_{1}, \dots, x_{n}$ are replaced by $X_{1}, \dots, X_{n}$. 

If the loss function is squared integrable: $\E[\Lf^{2}(X)] < \infty$, then the gradient estimator of our utility function converges almost surely: $\tnabla_{{\theta}^{n}}\widehat{U}  \asto \nabla \Uf(\theta)$, where $\asto$ represents almost sure convergence. 
\end{proposition}

\begin{proof}
Refer to \ref{proof2} in the supplementary material. This result was proved as early as 2011 by \cite{Ollivier_2017} although their final manuscript was published a few years latter. Our approach though is quite original as we use the fundamental property of almost sure convergence that implies also Cesaro convergence. As Monte Carlo estimates are in nature Cesaro sums, this remarkable property ensures convergence of the MC NGD algorithm. Intrinsically, it means our Monte Carlo estimate is consistent. This is the theoretical foundation for our algorithm. The square integrability condition $\E[\Lf^{2}(X)] < \infty$ is crucial to guarantee the convergence of the algorithm.
\end{proof}

\subsection{Convergence Properties}\label{sec:conv}
Since the MC NGD algorithm is quite complex, we are forced to simplify the problem and look at convergence properties only for the CF NGD algorithm. We additionally restrict ourselves to the canonical convex-quadratic function $f(x) = \frac{1}{2} x^\mathrm{T} \HMat x$, as introduced in \cite{Auger_2004}, where $\HMat$ is a positive definite symmetric matrix. We also make assumptions on our learning rates: there exist $\nummin > 0$ and $\nucmin > 0$ such that
\begin{gather}
\nummin \leq \nu_{\mu}^{t} \lambda_{1}( (\Sigma^{t})^{-1} \delta \Sigma^{t}) \leq 1,\label{equation:learning_rate_cond_mu}\\
\nucmin \leq \nu_{\Sigma}^{t} \lambda_{1}( (\Sigma^{t})^{-1} \delta \Sigma^{t}) \leq 1/2.\label{equation:learning_rate_cond_sigma}
\end{gather}
A way to satisfy these inequalities is to set $\nu_{\mu}^{t} = \alpha_{\mu} / \lambda_{1}((\Sigma^{t})^{-1} \delta \Sigma^{t})$  (respectively $\nu_{\Sigma}^{t} = \alpha_{\Sigma} /  \lambda_{1}((\Sigma^{t})^{-1} \delta \Sigma^{t})$ with $\alpha_{\mu}$ (resp. $\alpha_{\Sigma}$) being a learning rate constant for $\mu$ (resp. $C$) such that $0 < \alpha_{\mu} \leq 1$ (resp. $0 < \alpha_{\Sigma} \leq 1/2$).

With these assumptions, we can prove the global convergence of $\mu$ and $C$ with linear convergence speed as stated by the proposition~\ref{prop:convergence} below. The optimum of our quadratic test function is zero. In the following, we denote by $\norm{M}$ the Frobenius norm of $M$, namely $\norm{M} = \Tr^{1/2}(M^\trans M)$.

\begin{proposition}\label{prop:convergence}
For the canonical quadratic optimization, the natural gradient descent is explicit and given by 
\begin{align} 
\mu^{t+1}&= \mu^{t} - \nu_{\mu}^{t} \,  \kappa \, \Sigma^{t} \HMat \mu^{t} \nonumber \\
\text{and } \quad \Sigma^{t+1} &= \Sigma^{t} - \nu_{\Sigma}^{t} \, \kappa \, \Sigma^{t} \HMat \Sigma^{t} \label{equation:natural_gradientd_update_equations}
\end{align}
where 
\begin{equation}\label{equation:kappa}
\kappa =  \frac{2 \pi}{\left( \Gamma\left(\frac d 2 + 1 \right) \sqrt{\det(\HMat)} \right)^{2/d}}
\end{equation}
with $\Gamma(z) = \int_0^{+\infty}  t^{z-1}\,e^{-t}\,\mathrm{d}t$ the regular Gamma function.  Suppose \eqref{equation:learning_rate_cond_mu} and \eqref{equation:learning_rate_cond_sigma} hold, then 
$\norm{\mu^{t}}$ (resp. $\norm{\Sigma^t}$) converges to zero with a rate of convergence defined as the $\limsup$ of the ratio of two consecutive terms upper bounded by $1 - \nummin$ (respectively $1 - \nucmin$).
\end{proposition}
\begin{proof}
Refer to \ref{proof3}  in the supplementary material. The intuition of the result of the convergence appeared as early as in  \cite{Auger_2004} and was progressively improved and more detailled in  \cite{Arnold_2010},  \cite{Akimoto_2012b},  \cite{Beyer:2014} and lately  \cite{Hansen_2016} among others. The novelty of our proof is the explicit computation of the natural gradient descent update equations with the true computation of the constant $\kappa$ in the update equations \eqref{equation:natural_gradientd_update_equations}.
\end{proof}

\begin{remark}
Proposition~\ref{prop:convergence} is remarkable because it indicates that the rate of convergence is linear, which is unexpectedly fast for this kind of algorithms and explains why the CMA-ES is currently the state of the art solution for several complex optimization settings. 
\end{remark}

\section{Numerical Results} \label{NumericalResults}
In order to test the efficiency of our optimization procedure that results in a CMA-ES optimization for Learning parameters of our graphical model and the price target and stop loss, we look at a trend following algorithm where we enter a long (respectively short) trade if the dynamic Bayesian network forecast is above (respectively below) the previous day closing price. For each comparison, we add an offset $\mu$ to avoid triggering false alarm signals. We set for each trade a pre-determined profit and stop loss target in ticks. The chosen architecture for our graphical model is a combination of model of Figure~\ref{DBN2} and ~\ref{DBN3} and given in Figure~\ref{Our model}.

\begin{figure}[ht]
\centering
\resizebox {.47\textwidth} {!} {
\begin{tikzpicture}
 \node[box,draw=white!100] (Latent) {\textbf{Latent}};
 \node[main] (L1) [right=of Latent] {$\mathbf{x}_{11}$};
 \node[main] (L2) [right=of L1] {$\mathbf{x}_{21}$};
 \node[box,draw=white!100] (L4) [right=1.5cm of L2] {};
 \node[main] (Lt) [right=of L4] {$\mathbf{x}_{t1}$};
 \node[main] (L21) [below=0.5cm of L1 ] {$\mathbf{x}_{12}$};
 \node[main] (L22) [below=0.5cm of L2] {$\mathbf{x}_{22}$};
 \node[box,draw=white!100] (L24) [below=1.1cm of L4] {};
 \node[main] (L2t) [below=0.5cm of Lt] {$\mathbf{x}_{t2}$};
 \node[main,fill=black!35] (O1) [below=of L21] {$\mathbf{z}_1$};
 \node[main,fill=black!35] (O2) [below=of L22] {$\mathbf{z}_2$};
 \node[box,draw=white!100] (O4) [below=1.3cm of L24] {};
 \node[main,fill=black!35] (Ot) [below=of L2t] {$\mathbf{z}_t$};
 \node[box,draw=white!100,below=90pt] (Observed) {\textbf{Observed}};

 \path (L1) edge [connect] (L2)
        (L2) -- node{\ldots} (L4)
        (L4) edge [connect] (Lt);
 \path (L21) edge [connect] (L22)
        (L22) -- node{\ldots} (L24)
        (L24) edge [connect] (L2t);
 \path (L21) edge [connect] (O1)
	(L22) edge [connect] (O2)
	(L2t) edge [connect] (Ot);
 \path (L1) edge [bend right=22,connect] (O1)
	(L2) edge [bend right=22, connect] (O2)
	(Lt) edge [bend right=22, connect] (Ot);
 \path (L1) edge [connect] (L22)
        (L4) edge [connect] (L2t);
 \path (O1) edge [connect] (L2)
	 (O1) edge [connect] (L22)
        (O4) edge [connect] (Lt)
        (O4) edge [connect] (L2t);
 \draw [dashed, shorten >=-0.5cm, shorten <=-2.5cm]
     ($(O1)+(0,0.855)$) -- ($(Ot)+(0,0.8)$);
\end{tikzpicture}}
\caption{Combination of a multi state Kalman filter and echo neural networks (ESN) with feedback effet of observable variable on next latent variables}\label{Our model}
\end{figure}
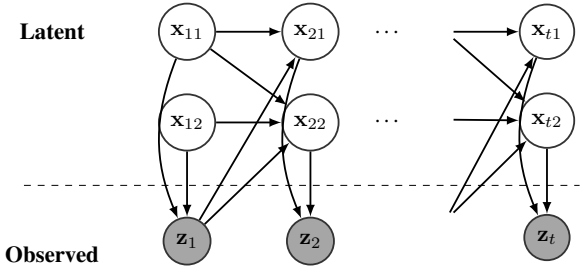

The corresponding dynamic is composed of two latent variables representing short and long term effects, denote respectively: $x_{.1}$ and $x_{.2}$. We assume that long term effect impacts the next long term latent variable while there is no influence of short term latent variable on the next long term effect latent variable. This results in the following equations:
\begin{align}
   \mathbf{x}_{t+1} & = \mathbf{\Phi}    \mathbf{x}_t + \mathbf{c}_t +  \mathbf{w}_t \label{model_eq1} \\
   \mathbf{z}_t        & = \mathbf{H} \mathbf{x}_t +  \mathbf{v}_t \label{model_eq2} 
\end{align}   

We assume that the observation noise $\mathbf{v}_t$ follows a multivariate normal distribution with zero mean and covariance matrix given by $\mathbf{R}_t$: $\mathbf{v}_t \sim \mathcal{N}\left(0, \mathbf{R}_t\right)$. In addition, the initial state, and noise vectors at each step ${\mathbf{x}_0, \mathbf{w}_1, \ldots, \mathbf{w}_t, \mathbf{v}_1, \ldots, \mathbf{v}_t}$ are assumed to be all mutually independent. We also denote by $\mathbf{P}_t=\operatorname{Cov}(\mathbf{x}_{t})$ the covariance matrix of $\mathbf{x}_{t}$. We assume the following parameters:
\vspace{-0.25cm}
\begin{align}
& \mathbf{\Phi}(x) \! =\!  \left[ {\begin{array}{cc} \!\!  p_1 & p_2 \! \!  \\ \!\!   0 &  p_3 \!\!   \end{array} } \right]\! \!  , 
\mathbf{H} \! =\!  \left[ {\begin{array}{c} \!\! p_4 \!\!  \\  \!\!  p_5 \! \! \end{array} } \right]\! \! ,
\mathbf{Q}_{t=0} \! = \!  \left[ {\begin{array}{l}  \!  p_6^2  \; p_6p_7 \!  \\ \!  p_7 p_6 \;   p_8^2  \!  \end{array} } \right],  \nonumber  \\
& \mathbf{R}_{t=0} \! = \! \left[ {\begin{array}{c} \!\!\!\!  p_9 \!\!\!\!    \end{array} } \right]\! \!, \! 
\mathbf{P}_{t=0}  \! =\!   \left[ {\begin{array}{l} \!\!\!\!  p_{10}  \; 0 \!\!\!\!  \\ \!\!\!\! 0 \;   p_{11}  \!\!\!\!\end{array} } \right]\! \! , \mathbf{c}_t \! = \!  \left[ {\begin{array}{c} \!\!\!\! p_{12} (.5\! +\! p_{13} \! -\!  k_t  )\!\!\!\! \\ \!\!\!\! p_{14}  (.5\! +\!  p_{15} \! -\!  k_t  )\!\!\!\!   \end{array} }  \right] \qquad \label{model_eq3} 
\end{align}
The variable $k_t$ is computed as the ratio of the difference between the last price observation and the running minimum over the difference between the running maximum and minimum observed over 20 periods. The pseudo code is given in Algorithm~\ref{algo:pseudo_code}.
\begin{algorithm}
\caption{Graphical model Trend following algorithm}\label{algo:pseudo_code}
	\begin{algorithmic} 
	\STATE \textbf{Initialize common trade details}
	\STATE SetProfitTarget( target)							
	\STATE SetStopLoss( stop\_loss )							
	\\
	\WHILE{ Not In Position}											
		\IF{ DBN( $p_1, \ldots, p_n$).Predict[0] $\ge$ Close[0] + $\mu$} 	 
			\STATE EnterLong()										
		\ELSIF{ DBN( $p_1, \ldots, p_n$).Predict[0] $\le$ Close[0] + $\mu$} 	
			\STATE EnterShort()										
		\ENDIF
	\ENDWHILE
	\end{algorithmic}
\end{algorithm}

Our resulting algorithm depends on the following parameters $p_1, \ldots, p_n$, the parameters of our dynamic graphical model, the profit target, the stop loss and the signal threshold $\mu$. We could estimate the graphical model parameters with the EM procedure, then optimize the profit target, stop loss and signal threshold separately. However, if by bad luck, the dynamics of the graphical model is incorrect, the noise induced by wrong specification will be factored in these three parameters.

We opt rather for a combined optimization of all the parameters. We use one minute data for the S\&P 500 index futures from 01Jan2017 to 01Jan2018 but take daily decisions. We train our model on the first 6 months representing about 170,000 data points and test it on the next six months. For a model given by equations \eqref{model_eq1} and \eqref{model_eq2}  and parameters specified in \eqref{model_eq3}, the optimization encompasses 18 parameters: $p_1, \ldots, p_{15}$, the profit target, the stop loss and the signal offset $\mu$, making it non trivial. We use the CMA-ES algorithm to find the optimal solution. In our optimization, we add a $L1$ penalty to enforce sparsity of our parameters. Obtained parameters are summarized in Table~\ref{tab:param}. Out of 18 parameters, only 10 are not null. Table~\ref{tab:model_stats} shows that statistics for train and test are very similar, which is a good sign for little overfitting. Figures \ref{fig:train} and \ref{fig:test} are very similar confirming our intuition of reduced overfitting.

\begin{figure}[!ht]
\centering
\includegraphics[width=6.5cm]{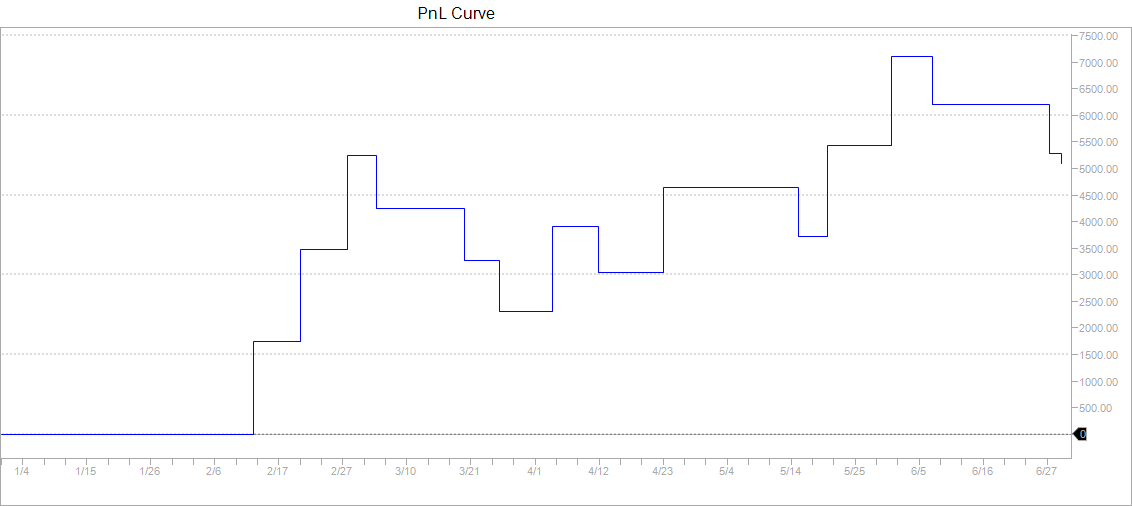}
\caption{Graphical model algorithm on train data set}
\label{fig:train}
\end{figure} 
\begin{figure}[!ht]
\centering
\includegraphics[width=6.5cm]{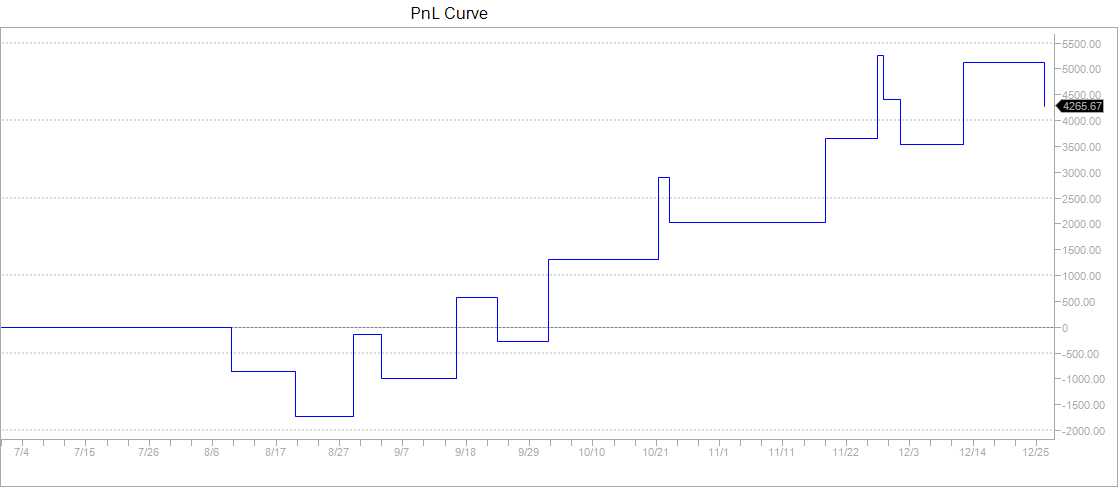}
\caption{Graphical model algorithm on test data set}
\label{fig:test}
\end{figure}

\begin{table}[!ht]
  \centering
  \caption{Optimal non null parameters}
  \resizebox{0.47 \textwidth}{!}{
    \begin{tabular}{|c|r|r|r|r|r|r|r|r|r|r|r|r|r|r|r|r|r|r|}
    \toprule
   Parameters 
& \multicolumn{1}{l|}{$p_{01}$} 
& \multicolumn{1}{l|}{$p_{03}$}
& \multicolumn{1}{l|}{$p_{04}$} 
& \multicolumn{1}{l|}{$p_{05}$} 
& \multicolumn{1}{l|}{$p_{06}$} 
& \multicolumn{1}{l|}{$p_{07}$} 
& \multicolumn{1}{l|}{$p_{12}$} 
& \multicolumn{1}{l|}{threshold} 
& \multicolumn{1}{l|}{stop} 
& \multicolumn{1}{l|}{target} \\
    \midrule
Value 
& 24.8  
& 11.8  
& 46.2 
 & 77.5 
 & 67    
& 99.9   
& 99.9   
& 5     
& 80    
& 150 \\
    \bottomrule
    \end{tabular}}%
  \label{tab:param}%
\end{table} 
We compare our algorithm with a classical EM  and a traditional Moving Average (MA) crossover approach to test the efficiency of our graphical model for trend detection. The MA algorithm generates a buy (resp. sell) signal when the fast moving average crosses above (resp. below) the long moving average.  Table~\ref{tab:mavskf} details the results of this comparison. We see that on the train dataset, the three algorithms share similar performances (similar first two columns) but they perform quite differently on the test set and only our graphical model manages to have a Sharpe ratio over 1 on the test period. Overall, the MA algorithm seems to suffer from overfitting as it performance in the test set is very different from the one on the train. This is confirmed by Figure \ref{fig:ma_train} and \ref{fig:ma_test}, with sharp contrast between train and test. 

\begin{figure}[!ht]
\centering
\includegraphics[width=7cm]{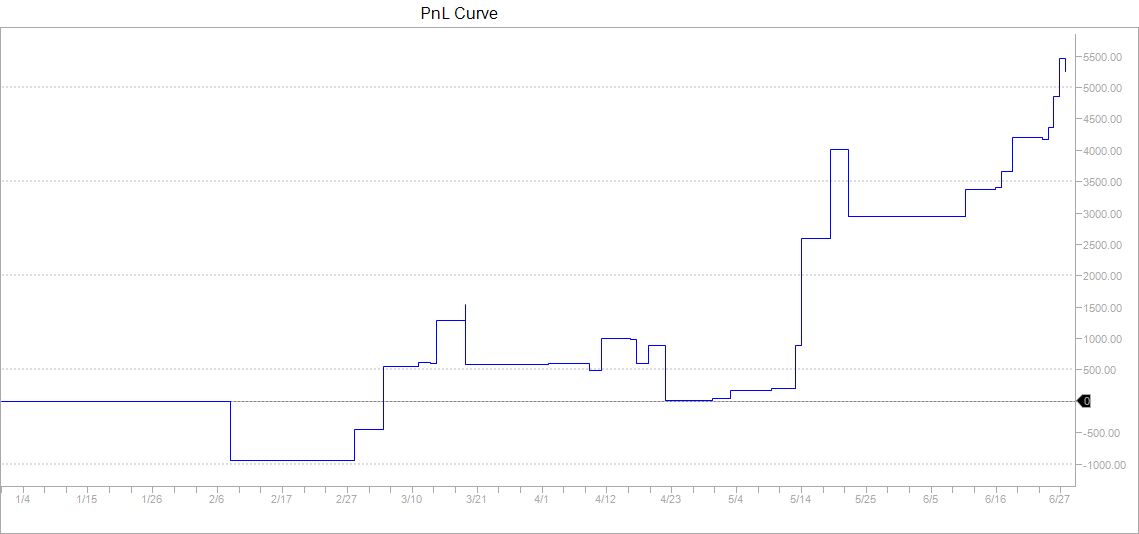}
\caption{Moving Average Crossover algorithm on train data set}
\label{fig:ma_train}
\end{figure}

\begin{figure}[!ht]
\centering
\includegraphics[width=7cm]{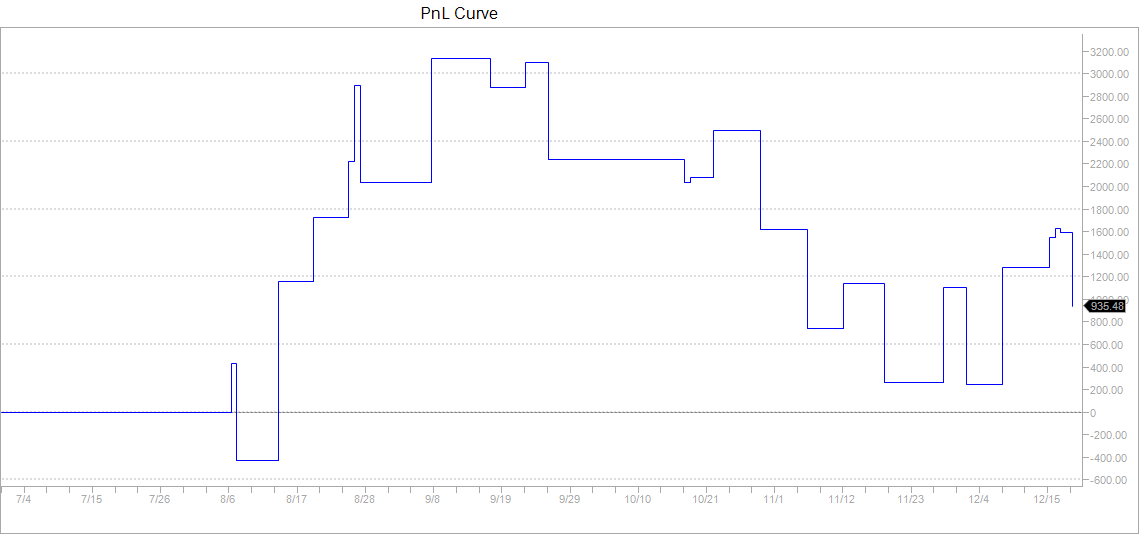}
\caption{Moving Average Crossover algorithm on test data set}
\label{fig:ma_test}
\end{figure}

\begin{table}[!ht]
  \centering
     \caption{Train/Test statistics for NGD graphical model}
    \resizebox{0.5 \textwidth}{!}{
    \begin{tabular}{|c|c|c|c|c|}
    \toprule
    \rowcolor{LightGray} Performance & Net Profit  & Avg. Trade & Tot. Net Profit (\%) & Ann. Net Profit (\%) \\
    \midrule
    Train &\textcolor[rgb]{0,.502, 0}{5,086 \euro}  & \textcolor[rgb]{0,.502, 0}{339.05 \euro} & \textcolor[rgb]{0,.502, 0}{5.09\%} 
		& \textcolor[rgb]{0,.502, 0}{10.59\%} \\
    \midrule
    Test  & \textcolor[rgb]{0,.502, 0}{4,266 \euro} & \textcolor[rgb]{0,.502, 0}{284.38 \euro} & \textcolor[rgb]{0,.502, 0}{4.27\%} 
		& \textcolor[rgb]{0,.502, 0}{8.69\%} \\
    \midrule
    \rowcolor{LightGray} Performance & Vol   &Sharpe Ratio & Max. Drawdown & Recovery Factor \\
    \midrule
    Train & \textcolor[rgb]{0,.502, 0}{6.54\%} & \textcolor[rgb]{0,.502, 0}{1.62} & {-2,941 \euro} & \textcolor[rgb]{0,.502, 0}{3.510} \\
    \midrule
    Test  & \textcolor[rgb]{0,.502, 0}{6.20\%} & \textcolor[rgb]{0,.502, 0}{1.40} & {-1,721 \euro} & \textcolor[rgb]{0,.502, 0}{4.948}\\
    \midrule
    \rowcolor{LightGray} Performance &  Profit Factor &  \# of Trades & \# of Contracts Avg. Winning Trade & Max. conseq. Winners \\
    \midrule
    Train &  1.75 \euro & 15  & \textcolor[rgb]{0,.502, 0}{1,692.09 \euro} & 3 \\
    \midrule
    Test  & 1.62 \euro & 15  & \textcolor[rgb]{0,.502, 0}{1,588.92 \euro} & 2 \\
    \bottomrule
    \end{tabular}}
  \label{tab:model_stats}
\end{table}

\begin{table}[!ht]
  \centering
  \caption{Graphical Model versus EM and MA crossover}
\resizebox{0.5 \textwidth}{!}{
    \begin{tabular}{|c|c|c|c|c|}
    \toprule
    \rowcolor{LightGray}  Algo  &  Train: Percent Profitable & Train: Total Net Profit & Test: Total Net Profit & Test: Recovery Factor \\
    \midrule
    MA Cross over 		& \textcolor[rgb]{0,.502,0}{54 \%}  	& \textcolor[rgb]{0,.502,0}{5,260.00  \euro}	& {935 \euro} & 0.32  \\
    \midrule
    EM optimized model & \textcolor[rgb]{0,.502,0}{42 \%} & \textcolor[rgb]{0,.502, 0}{5,165.72 \euro}		& {2,250 \euro} & 0.56 \\
    \midrule
    NGD Graphical model& \textcolor[rgb]{0,.502,0}{48 \% } & \textcolor[rgb]{0,.502, 0}{5,085.79  \euro} 	& \textcolor[rgb]{0,.502, 0}{4,266 \euro} & \textcolor[rgb]{0,.502, 0}{2.48} \\
    \midrule
    \rowcolor{LightGray} Algo  & Test: Sharpe Ratio & Total \# of Trades & Overall Profit Factor & Test: Max. Drawdown  \\
    \midrule
    MA Cross over          & \textcolor[rgb]{1,0,0}{0.41}		& 26  & 1.13  & \textcolor[rgb]{1,0,0}{-\euro 2,889} \\
    \midrule
    EM optimized model   & \textcolor[rgb]{1,0,0}{0.85}		& 28 & 1.25	& \textcolor[rgb]{1,0,0}{-\euro 2,523 } \\
    \midrule
    NGD Graphical model &  \textcolor[rgb]{0,.502,0}{1.40  }	& 30  & 1.62  & {-\euro 1,721} \\
    \bottomrule
    \end{tabular}}%
  \label{tab:mavskf}%
\end{table}

\section{Conclusion}
In this paper, we presented a novel approach for learn-ing parameters in dynamic graphical models that tackles the targeted usage of our dynamic graphical model and incorporates them in the natural gradient optimization of the loss function. Compared to EM method, this method performs better in real scenarios and is less prone to overfitting.

{\scriptsize
	\bibliographystyle{apalike}
	\bibliography{mybib}
}

\appendix
\clearpage

\line(1,0){225}\\\\
{\Huge
{\textbf{Supplementary}}\\
{\textbf{Materials}}
}
\line(1,0){225}

\section{Proof of proposition \ref{prop:natural_gradient}}\label{proof0}
\begin{proof}
The regularity condition that the expectation of our loss function is square integrable ($\EX{\Lf^{2}(X)} < \infty$) implies that for any $d \geq 1$, we also have
$\EX{\Lf^{2/d}(X)} < \infty$. 

Hence we can interchange the order of integration and differentiation and compute the natural gradient update as 
\begin{equation}\label{equation:nabla_u_score}
\nabla \Uf(\theta) = \EX{\Lf^{2/d}(X) S(\theta, x)},
\end{equation}

where $S(\theta, x) = \nabla \ln l(\theta, x)$ is the score function, defined as the gradient of the log-likelihood $L(\theta, x)  = \ln l (\theta, x) = \ln p_{\theta}(x)$ of the multivariate Gaussian distribution with respect to its moment parameters. 

For a multi variate normal distribution, the score function is given by
\begin{equation*}
 S(\theta, x) \!  = \!\! \begin{bmatrix}
\Sigma^{-1}(x - \mu)\\
-\frac{1}{2}\vect (\Sigma^{-1} - \Sigma^{-1}(x - \mu)(x - \mu)^\mathrm{T}\Sigma^{-1} )
\end{bmatrix}  
\end{equation*}

For a multi variate normal distribution, the Fisher matrix information has a closed form given by
\begin{equation}
\FIM = \begin{bmatrix}
\Sigma^{-1} & 0 \\ 0 & \frac{1}{2} ( \Sigma^{-1} \otimes \Sigma^{-1})
\end{bmatrix},\label{equation:fisher}
\end{equation}
where $\otimes$ denotes the Kronecker product operator. The natural gradient writes 
\begin{equation}\label{equation:update_proof1}
\theta^{t+1} = \theta^{t} - \nu_{t} \FM_{\theta^{t}}^{-1} \nabla \Uf(\theta^{t})
\end{equation}

Combining \eqref{equation:nabla_u_score}, \eqref{equation:fisher} and \eqref{equation:update_proof1}, and splitting the learning rate $\nu_{t}$ among the mean vector and the covariance matrix, with two different learning rates, we obtained 
\begin{equation}\label{equation:natural_gradient_proof}
\begin{split}
\mu^{t+1} = \mu^{t} - \nu_{\mu}^{t} \delta m^{t}, \\
\Sigma^{t+1} = \Sigma^{t} - \nu_{\Sigma}^{t} \delta \Sigma^{t}\enspace.
\end{split}
\end{equation}

with
\begin{equation*}
\begin{split}
\delta \mu^{t} &= \E_{X \sim P_{\theta^{t}}}[\Lf^{2/d}(X) (X - \mu^{t})] \\
\delta \Sigma^{t} &= \E_{X \sim P_{\theta^{t}}}\bigl[ \Lf^{2/d}(X) \bigl((X - \mu^{t})(X - \mu^{t})^\trans - \Sigma^{t} \bigr) \bigr]
\end{split}
\end{equation*}
which concludes the proof.
\end{proof}

\section{Proof of proposition \ref{prop:positivity}}\label{proof1}
\begin{proof}
The symmetry is obvious. Let us prove that the matrix $\Sigma_t$ is positive definite for each $t$ by mathematical induction. 
The result is trivially true for $t=0$ as $\Sigma_0$ is  positive definite. Suppose that at step $t$ the result is true and that the matrix $C^t$ is positive definite and symmetric. 
Notice that the update \ref{equation:natural_gradient} writes 
\begin{equation}\label{equation:proof1_eqCF}
\Sigma^{t+1} = \sqrt{\Sigma^{t}}  (I_{d}-\nu_{\Sigma}^{t} \sqrt{\Sigma^{t}}^{-1} \delta \Sigma^{t}   \sqrt{\Sigma^{t}}^{-1}) \sqrt{\Sigma^{t}}
\end{equation}
for the closed form (C) NGD and 
\begin{equation}\label{equation:proof1_eqMC}
\Sigma^{t+1} = \sqrt{\Sigma^{t}} ((I_{d} - \nu_{\Sigma}  \sqrt{\Sigma^{t}}^{-1}  \widehat{\delta \Sigma^{t}} \sqrt{\Sigma^{t}}^{-1} ) \sqrt{\Sigma^{t}}
\end{equation}
for the Monte Carlo case.  
Using the result that if $A$ and $B$ are symmetric positive definite, then $\sqrt{A} B \sqrt{A}$ is also positive definite, we get that $\Sigma^{t+1}$ stays positive definite. 
This concludes the proof that the matrix $\Sigma_t$ is positive definite for each $t$.
Should the learning rate be too large so that the inner term in \ref{equation:proof1_eqCF} $I_{d}-\nu_{\Sigma}^{t} \sqrt{\Sigma^{t}}^{-1} \delta \Sigma^{t}   \sqrt{\Sigma^{t}}^{-1}$ for the closed form case, and $I_{d} - \nu_{\Sigma}  \sqrt{\Sigma^{t}}^{-1}  \widehat{\delta \Sigma^{t}} \sqrt{\Sigma^{t}}^{-1} $ for the Monte Carlo case are not definite positive, the next iteration $\Sigma^{t+1}$ would also be not definitive positive. Should the previous matrices be non positive, the next iteration $\Sigma^{t+1}$ would non positive, which concludes the proof.
\end{proof}

\section{Proof of proposition \ref{prop:consistency}}\label{proof2}
We will decompose this proof into various elementary lemmas. Usually, when proving convergence for Monte Carlo, one relies on traditional statistical tools like the strong law of large numbers as well as some measure tools like Lebesgue's dominated convergence, Cauchy Schwartz inequality, etc... However, we take an original approach and rather use connection between single term convergence and Cesaro convergence. A first lemma that is intuitive is that almost sure convergence implies almost sure (a.s.) Cesaro convergence. To make things more precise, let $X$ be a random variable of $\R^{d}$ and $f_n: \R^{d} \rightarrow \R$ a series of real valued functions.
\begin{lemma}\label{proof2:Cesaro convergence}
If $\left( f_n \right)_{n \in \mathbb{N}}$ converges almost surely to 0, then its Cesaro sum obtained for $n$ independent realizations of $X$ converges almost surely to 0. In other words,
$$
f_n(X) \xrightarrow{a.s.}0\,\,\implies \,\,\frac{1}{n} \sum_{i=1}^{n}f_n(X_i)  \xrightarrow{a.s.}0
$$
where $X_i \sim X$ are i.i.d. variables.
\end{lemma}

\begin{proof}
For $m \leq n$, we have
\begin{equation}
\begin{split}
\lim_{n\to\infty} \Bigl\lvert \frac{1}{n}\sum_{i=1}^{n} f_{n}(X_i) \Bigr\lvert \leq \lim_{n\to\infty} \frac{1}{n}\sum_{i=1}^{n} \sup_{j \geq n} \abs{f_{j}(X_i)} \\
\leq \lim_{n\to\infty} \frac{1}{n}\sum_{i=1}^{n} \sup_{j \geq m} \abs{f_{j}(X_{i})}.
\end{split}\label{equation:suph}
\end{equation}

By assumption, $f_{n}(x) \asto 0$ almost everywhere in $x$ as $n$ tends to infinity. This implies that $\sup_{j \geq m} \abs{f_{j}(x)} \asto 0$ almost everywhere in $x$ as $\mu$ tends to infinity. 
We can apply the Lebesgue's dominated convergence theorem to prove that $\E[\sup_{j \geq m} \abs{f_{j}(X)}] \to 0$ as $m \to \infty$. 
This implies that $\E[\sup_{j \geq m} \abs{f_{j}(X)}]$ is finite for $\mu$ sufficiently large. The strong law of large numbers (\lln) applies and proves that the right hand side of \eqref{equation:suph} converges to $\E[\sup_{j \geq m} \abs{f_{j}(X)}]$ when $n$ becomes large. The last expectation converges to $0$ when $\mu$ grows to infinity. This concludes the proof.
\end{proof}

\begin{remark}
That convergence implies Cesaro convergence relies strongly on the almost sure convergence nature. Should the initial convergence only be in probability, the convergence of our series of functions $\left( f_n \right)_{n \in \mathbb{N}}$ would not imply Cesaro convergence. Take for instance a standard Gaussian variable $X \sim \mathcal{N}(0,1)$ and the series of function$\left( f_n \right)_{n \in \mathbb{N}}$ where $f_n$ is equal to 0 everywhere except when $X \geq n$ and takes the value $1 / (1-\Phi(n))$ where $\Phi$ is the cumulative density function (c.d.f). 
Take $X_1, X_2, ...$ n independent standard Gaussian variables, it is easy to see that the Cesaro sum $\frac{\sum_{k=1}^n f_n(X_k)}{n}$ does not converge to 0 as its mean is by construction equal to 1, while $(f_n)$ converges in probability to 0. 
This is because the function $f_n$ takes non zero value on a smaller and smaller interval measure: $X \geq n$ but with a value that explodes and is given by $1 / (1-\Phi(n))$ to compensate for the reduced interval on which the function $f_n$ is not null.
\end{remark}

The second lemma concerns the almost sure consistency of our estimator of the loss function. 
\begin{lemma}\label{proof2:consistency}
$\widehat{\Lf^{n}}^{2/d}(x)$ converges a.s. to $\Lf^{2/d}(x)$
\end{lemma}

\begin{proof}
The estimator $\widehat{\Lf^{n}}^{2/d}(x)$ writes as
\begin{equation*}
\widehat{\Lf^{n}}^{2/d}(x) = \biggl( \frac{1}{n}\sum_{j = 1}^{n} \frac{\mathbf{1}_{\{f(x_{j}) \leq f(x)\}}}{p_{\theta^{t}}(x_{j})}\biggr)^{2/d}.
\end{equation*}
hence, we have 
\begin{align*}
\lim_{n\to\infty}\widehat{\Lf^{n}}^{2/d}(x) &= \left(\lim_{n\to\infty} \frac{1}{n}\sum_{j = 1}^{n} \frac{\mathbf{1}_{f(X_{j}) \leq f(x)}}{p_{\theta}(X_{j})}\right)^{2/d}.
\end{align*}
We know that $\E[\Lf^{2}(X)]  < \infty$ which implies $\E[\Lf^{2/d}(X)]  < \infty$ for any dimension $d$. This implies in particular taht $\Lf^{2/d}(X)= (\leb[y: f(y) \leq f(x)] )^{2/d}$ is finite almost everywhere. We can safely apply the strong law of large numbers to get
\begin{align*}
\lim_{n\to\infty}\widehat{\Lf^{n}}^{2/d}(x) &= (\leb[y: f(y) \leq f(x)] )^{2/d} = \Lf^{2/d}(x)
\end{align*}
almost surely and almost everywhere in $x$, which concludes the proof.
\end{proof}

We are now ready to prove proposition \ref{prop:consistency}.
\begin{proof}
The assumption in proposition \ref{prop:consistency} is that $\Lf(X)$ is square integrable, that is $\E[\Lf^{2}(X)] < \infty$, hence $\Lf^{2/d}(X)$ is also integrable, which leads trivially that the expectation of $\Lf^{2/d}(X)$ is finite 
\begin{equation}
\E[\Lf^{2/d}(X)]  \leq  \infty. \label{equation:bounded_expectation_2}
\end{equation}
The variance of $\Lf^{2/d}(X)$ being non negative, we have the trivial inequality between the first two moments: 
$$\Big(\E[\Lf^{2/d}(X)]\Big) ^{2} \leq \E[\Lf^{4/d}(X)].$$ 
Moreover, Cauchy-Schwarz inequality gives us that 
$$\E[\norm{\Lf^{2/d}(X) \tnabla L(\theta, x)}]^{2} < \E[\Lf^{4/d}(X)]\E[\norm{\tnabla L(\theta, x)}^{2}]. $$
Using the relationship between the variance of the score and the information Fisher, we can notice that $\E[\norm{\tnabla L(\theta, x)}^{2}] = \Tr(\FIM^{-1}) < \infty$. 
This implies also that the expectation of $\norm{\Lf^{2/d}(X) \tnabla L(\theta, x)}$ is finite
\begin{equation}
\E[\norm{\Lf^{2/d}(X) \tnabla L(\theta, x)}] < \infty \label{equation:bounded_expectation_1}\\
\end{equation}

Let us take the function series $(f_n)$ as the difference between our estimated loss function and the loss function:
$$
f_{n}(x) = \widehat{\Lf^{n}}^{2/d}(x) - \Lf^{2/d}(X).
$$

Let us decompose $\tnabla_{{\theta}^{n}}\widehat{U}$ in three terms as follows:
\begin{eqnarray}
\tnabla_{{\theta}^{n}}\widehat{U} \!\!\! &=& \!\! \!  \frac{1}{n}\sum_{i=1}^{n} \Lf^{2/d}(X_{i}) \tnabla L(\theta; X_{i}) \nonumber \\
& & - \biggl( \underbrace{\frac{1}{n}\sum_{i=1}^{n} \widehat{\Lf^{n}}^{2/d}(X_{i})}_{= b}\biggr) \biggl( \frac{1}{n} \sum_{i=1}^{n} \tnabla L\theta; X_{i})\biggr) \nonumber \\
& & + \frac{1}{n}\sum_{i=1}^{n} f_{n}(X_{i}) \tnabla L(\theta; X_{i}) \label{equation:propconv}
\end{eqnarray}

We want to prove that $\tnabla_{{\theta}^{n}}\widehat{U}  \asto \nabla \Uf(\theta)$ almost surely. We can examine successively the three terms.

Using \eqref{equation:bounded_expectation_1} and the strong law of large numbers, the first term converges to $\E[\Lf^{2/d}(X) \tnabla L(\theta, x)] = \nabla \Uf(\theta)$ almost surely as $n$ tends to infinity. 

We now examine the second and third terms in \eqref{equation:propconv}. To conclude the proof, we need to prove that these two terms converge a.s. to zero.

Let us tackle the second term. By lemma \ref{proof2:consistency}, we know that our estimated loss function is almost surely consistent, which implies that $f_n(x)$ converges almost surely to $0$. By lemma \ref{proof2:consistency}, this implies that we have the almost surely Cesaro convergence of 

\begin{equation}
\frac{1}{n}\sum_{i=1}^{n} f_{n}(X_{i}) \asto 0 \quad \text{as } n \to \infty. \label{equation:prop1}
\end{equation}

The almost sure convergence \eqref{equation:prop1} implies that the limit $\lim_{n\to\infty}\sum_{i=1}^{n} \widehat{\Lf^{n}}^{2/d}(X_{i})/n$ agrees with $\lim_{n\to\infty}\sum_{i=1}^{n} \Lf^{2/d}(X_{i})/n$. We also have from \eqref{equation:bounded_expectation_2} and from the strong law of large numbers that the Cesaro sum $\sum_{i=1}^{n} \widehat{\Lf^{n}}^{2/d}(X_{i})/n $ converges almost surely to the expectation $\E[\Lf^{2/d}(X)]$.
The latter is finite by assumption. Also, by the strong law of large numbers we have that the Cesaro sum $\sum_{i=1}^{n} \tnabla L(\theta; X_{i})/n$ also converges to $0$ when $n$ tends to infinity. We can conclude that the second term of \eqref{equation:propconv} converges to zero almost surely.

We are now left with the third term of \eqref{equation:propconv}. 
The proof of its almost sure convergence to zero is similar. 
A Cauchy-Schwarz inequality applied to the second term of \eqref{equation:propconv} gives us a way to control the third term as follows:
\begin{multline*}
\!\!\!\! \!\!   \left\lvert \sum_{i=1}^{n} \frac{f_{n}(X_{i}) \tnabla L(\theta; X_{i})}{n} \right\rvert^{2} \!\!\!\! \leq  \sum_{i=1}^{n} \frac{f_{n}(X_{i})^{2}}{n} \sum_{i=1}^{n} \frac{\norm{\tnabla L(\theta; X_{i})}^{2}}{n}.
\end{multline*}
We can take of the two terms in this inequality successfully. The strong law of large numbers proved that the second term of the right hand side converges to the expectation $\E[\norm{\tnabla L(\theta, x)}^{2}] $, which can be expressed in terms of the trace of the inverse of the Fisher information matrix: $\Tr(\FIM^{-1})$ which is finite. This means that the second term of the right hand side converges to a constant. We are left to prove that the first term on the right hand side converges to zero almost surely. By lemma \ref{proof2:consistency}, we have that $f_n \asto 0$  almost everywhere in $x$ as $n \to \infty$, which implies that  $f_n^2 \asto 0$  almost everywhere in $x$ as $n \to \infty$. By lemma \ref{proof2:Cesaro convergence}, this implies the Cesaro convergence of $\sum_{i=1}^{n} \frac{f_{n}(X_{i})^{2}}{n}$ to 0  almost everywhere in $x$ as $n \to \infty$ which concludes the proof.
\end{proof}

\section{Proof of proposition \ref{prop:convergence}}\label{proof3}
Like for proposition \ref{prop:consistency}, we will first prove a few lemma and the result will become easy. 
Since our objective function is explicit and given by a parabola, which is a very simple function, we can explicitly compute the natural gradient ascent. This is the subject of the first lemma:

\begin{lemma}\label{lemma:ng_algo}
The closed-form NGD algorithm on the benchmark quadratic function $\frac{1}{2} x^{\trans} \HMat x$ writes as
\begin{align}
\mu^{t+1} &= \mu^{t} - \nu_{\mu}^{t} \delta \mu^{t},  \label{equation:mu-quad}\\
 \text{with } \quad \, \delta \mu^{t} &= \kappa \,\, \Sigma^{t} \HMat \mu^{t} \label{equation:mu-quad2} \\
\text{and } \quad \Sigma^{t+1} &= \Sigma^{t} - \nu_{\Sigma}^{t} \delta \Sigma^{t}, \label{equation:sigma-quad} \\
 \text{with } \quad  \,\, \delta \Sigma^{t} &= \kappa \,\, \Sigma^{t} \HMat \Sigma^{t} \label{equation:sigma-quad2},
\end{align}
with the constant $\kappa > 0$ given by 
$$
\kappa =  \frac{2 \pi}{\left( \Gamma\left(\frac d 2 + 1 \right) \sqrt{\det(\HMat)} \right)^{2/d}}
$$
where $\Gamma(z)$ the Gamma function.
\end{lemma}

\begin{proof}
The natural gradient is directly related to the inverse Fisher information matrix times the gradient of the loss function: 
\begin{align}
\mu^{t+1} &= \mu^{t} - \nu_{\mu}^{t}  (\FIM^{-1} \nabla \Uf(\theta) )_{\mu} \\
\Sigma^{t+1} &= \Sigma^{t} - \nu_{\Sigma}^{t} (\FIM^{-1} \nabla \Uf(\theta) )_{\Sigma}
\end{align}

To fully make explicit the above formula, we need to compute explicitly $\FIM^{-1} \nabla \Uf(\theta)$ and  prove that it is given by
\begin{equation}\label{equation:explicit_formula}
\!\! \FIM^{-1} \nabla \Uf(\theta) =  \kappa
\begin{bmatrix}
\Sigma \HMat \mu \\
\vect(\Sigma \HMat \Sigma)
\end{bmatrix} .
\end{equation}

We can remark that $\leb[y: f(y) \leq f(x)]$ is equivalent to the volume of the ellipsoid given $\{y: y^\trans \HMat y \leq x^\trans \HMat x\}$, or equivalently, 
$\{ y: \ \frac{ \| \sqrt{\HMat} y \|}{\| \sqrt{\HMat} x \|} \leq 1 \}$. The last set is the image of unit ball $\{v :\|v \| \leq 1\}$ under the transformation given by $ \frac{\sqrt{\HMat} y}{\| \sqrt{\HMat} x \|} = v$ or equivalently $y = \| \sqrt{\HMat} x \| \sqrt{\HMat}^{-1} v$. The change of volume under this linear map is given by the determinant of the transform which is $\| \sqrt{\HMat} x \|^d \det(\HMat)^{-1/2}$, hence,
\begin{equation*}
\leb[y: f(y) \leq f(x)] = \| \sqrt{\HMat} x \|^d \det(\HMat)^{-1/2} V_{d}(1),
\end{equation*}
where $V_{d}(1)$ is the volume of the unit ball in dimension $d$. The utility or quasi objective function is defined as the expectation of the loss function to the power $2/d$. Hence, we get 
\begin{equation*}
\begin{split}
\Uf(\theta) :&= \EX{\Lf^{2/d}(X)}\\
&= \EX{ \left( \leb [y: f(y) \leq f(x)] \right) ^{2/d} }\\
& = \left( \frac{V_{d}(1)}{\sqrt{\det(\HMat)}} \right)^{2/d}  \E_{X\sim P_{\theta}}\bigl[ X^\trans \HMat X \bigr] \\
& = \left( \frac{V_{d}(1)}{\sqrt{\det(\HMat)}} \right)^{2/d}  ( \mu^\trans \HMat \mu + \Tr(\HMat \Sigma) ).
\end{split}
\end{equation*}

Differentiating the both side of the above relation, we have
\begin{equation*}
\nabla \Uf(\theta) = 2 \left( \frac{V_{d}(1)}{\sqrt{\det(\HMat)}} \right)^{2/d}
\begin{bmatrix}
\HMat \mu\\
\frac{1}{2} \vect(\HMat)
\end{bmatrix}.
\end{equation*}

We can use the fact that the volume of the unit ball is given by
$$
V_{d}(1)= \frac {\pi^{d/2}}{\Gamma\left(\frac d 2 + 1 \right)},
$$
to get
\begin{equation}\label{equation:nabla_u}
\nabla \Uf(\theta) = \kappa
\begin{bmatrix}
\HMat \mu\\
\frac{1}{2} \vect(\HMat)
\end{bmatrix}.
\end{equation}
with
$$
\kappa =  \frac{2 \pi}{\left( \Gamma\left(\frac d 2 + 1 \right) \sqrt{\det(\HMat)} \right)^{2/d}}
$$
For a multi variate Gaussian distribution, the inverse of the Fisher information matrix is given by
\begin{equation}
\FIM = \begin{bmatrix}
\Sigma & 0 \\ 0 &2  (\Sigma \otimes \Sigma)
\end{bmatrix},\label{equation:fisher_inv}
\end{equation}
where $\otimes$ denotes the Kronecker product operator. Multiplying the gradient of our utility function \eqref{equation:nabla_u}  by the inverse of the Fisher information matrix \eqref{equation:fisher_inv}, we get equation \eqref{equation:explicit_formula}, which concludes the proof.
\end{proof}

The second lemma relates to the condition number of a diagonal matrix series  $\left( \Lambda_{t} \right)_{t\in \mathbb{N}}$ that satisfies a particular recursive relationship given by
\begin{equation}
\Lambda_{t+1} = \Lambda_{t} - \nu_{t} (\Lambda_{t})^{2}.\label{equation:diag-up}
\end{equation}
We also take the convention that $ \lambda_{i}^t$ is the $i$th largest value of the matrix $\Lambda_{t}$. We recall that for a symmetric matrix $M$, its condition number denoted by $Cond(M)$ is defined by the ratio of its maximal to its minimal eigen value. For the specific case of a diagonal matrix, this is simply the ratio of its largest to its smallest value.

\begin{lemma}\label{lemma:condition_diagonal}
Let a diagonal matrix series $\left( \Lambda_{t} \right)_{t\in \mathbb{N}}$ that satisfies \eqref{equation:diag-up}, 
with the constraint on the learning rate at each step given by:
\begin{equation}
0 < \numin \leq \nu_{t} \lambda_{1}(\Lambda_{t}) \leq 1/2.\label{equation:learning_rate_cond_sigma2}
\end{equation} 
Then, the series of condition number of the matrix  $\Lambda_{t}$ satisfies for any $t > 0$ the following inequalities:
\begin{equation}
\frac{\Cond(\Lambda_{t+1}) - 1}{\Cond(\Lambda_{t}) - 1} \leq 1 - \numin. \label{equation:cond-im}
\end{equation}
\end{lemma}

\begin{proof}
We can remark that $\nu_{t}(\lambda_{i}^{t} + \lambda_{j}^{t}) \leq 1$ for any $i$, $j$. This is trivially obtained by the fact that as by definition, $\lambda_{1}$ is the largest eigen value, we have $\lambda_{i}^{t} \leq \lambda_{1}$ for any $i$, which combined with the constraint \eqref{equation:diag-up} implies $\nu_{t}(\lambda_{i}^{t} + \lambda_{j}^{t}) \leq 2 \nu_{t} \lambda_{1} \leq 1$.

Let suppose without loss of generality that we have sorted the diagonal matrix elements by their value so that the first element of the diagonal matrix value is the largest eigen values, the second element the second largest value, etc.. Let us now prove that the mapping between the matrix position and the eigen values ranking order stays the same between step $t$ and $t+1$. 

Let us prove the result by mathematical induction. 
The result is true for step $t=0$. Suppose the result is true at step $t$ and $\lambda_i^{t} \geq \lambda_j^{t}$. 
From \eqref{equation:diag-up} and the inequality $\nu_{t} (\lambda_{i}^{t} + \lambda_{j}^{t}) \leq 1$, we have
\begin{equation*}
\begin{split}
\lambda_{i}^{t+1} - \lambda_{j}^{t+1} &= \lambda_i^{t}(1-\nu_{t} \lambda_i^{t}) - \lambda_j^{t}(1-\nu_{t} \lambda_j^{t})\\
 &= (1 - \underbrace{ \nu_{t} (\lambda_{i}^{t} + \lambda_{j}^{t})}_{\leq 1})(\underbrace{\lambda_{i}^{t}-\lambda_{j}^{t}}_{\geq 0})\geq 0
\end{split}
\end{equation*}
with equality holding if and only if $\lambda_{i}^{t}=\lambda_{j}^{t}$. Therefore, if $\lambda_i^{t} > \lambda_j^{t}$, then $\lambda_i^{t+1} > \lambda_j^{t+1}$, which implies the result holds also at time $t+1$. 

As by convention there is mapping between diagonal element position and their eigen value rank, the condition number of matrix $\Cond(\Lambda_{t}) $ is computed at each step as the ratio of the first and last element of the diagonal matrix $\Lambda_{t}$ given by $\lambda_{1}^{t}/\lambda_{d}^{t}$. According to our recursive equation \eqref{equation:diag-up} we have
\begin{align}
\underbrace{\frac{\lambda_{1}^{t+1} - \lambda_{d}^{t+1}}{\lambda_{d}^{t+1}}}_{\Cond(\Lambda_{t+1})-1} &= \frac{\lambda_{1}^{t}(1 - \nu_{t}\lambda_{1}^{t}) - \lambda_{d}^{t}(1 - \nu_{t}\lambda_{d}^{t})}{\lambda_{d}^{t}(1 - \nu_{t}\lambda_{d}^{t})} \notag\\
&= \underbrace{\frac{(\lambda_{1}^{t} - \lambda_{d}^{t})}{\lambda_{d}^{t}}}_{\Cond(\Lambda_{t})-1} \frac{1 - \nu_{t}(\lambda_{1}^{t} + \lambda_{d}^{t})}{(1 - \nu_{t}\lambda_{d}^{t})}.\label{equation:cond-eq3}
\end{align}
We can notice that the second term in the right hand side is bounded as follows:
\begin{equation*}
\frac{1 - \nu_{t}(\lambda_{1}^{t} + \lambda_{d}^{t})}{1 - \nu_{t}\lambda_{d}^{t}} = 1 - \frac{\nu_{t} \lambda_{1}^{t}}{1 - \nu_{t} \lambda_{1}^{t} \Cond^{-1}(\Lambda_t)} \leq 1 - \nu_{t} \lambda_{1}^{t}.
\end{equation*}
Because of the lower bound on the learning rate in \eqref{equation:learning_rate_cond_sigma2}, the right term in the above inequality is bounded by $1 - \nucmin$. This leads to the result:
\begin{equation}
\frac{\Cond(\Lambda^{t+1}) - 1}{\Cond(\Lambda^{t}) - 1} \leq 1 - \nucmin.
\end{equation}
This concludes the proof.
\end{proof}

\begin{remark}
As the condition number defined as the ratio of the maximum with the minimum eigen value converges to 1, this implies that the diagonal matrix $\Lambda^t$ converges to the identity matrix in the sense of the Frobenius norm times its maximal eigen value (or equivalent its minimal eigen value as they get similar as $t$ tends to infinity). This is not by magic. The proportional convergence to the identity matrix is forced by the inequalities \eqref{equation:learning_rate_cond_sigma2}, satisfied by the learning rate at each step which forces all the matrix eigen values to progressively converge to a common number. We can also note that the largest eigen value converges to zero as $t$ tends to infinity and is decreasing. This means in particular that in order to satisfy the condition \eqref{equation:learning_rate_cond_sigma2}, the learning rate $\nu_t$ needs to be not too small and specifically not below ${\numin }/ {\lambda_{1}(\Lambda_{t})}$. This implies that the learning rate makes a trade-off between a small value to ensure that it is below $\frac{1}{2 \lambda_{1}(\Lambda_{t})}$ but above ${\numin }/ {\lambda_{1}(\Lambda_{t})}$, with the latter becoming larger and larger as $\lambda_{1}(\Lambda_{t})$ becomes smaller and smaller.
\end{remark}

A third lemma concerns the relationship between largest singular values and matrix norm. We denote by $\sigma_{i}(M)$ the $i$th largest singular value of a matrix $M$, and  $\sigma_{1}(M)$ its largest singular value.

\begin{lemma}\label{lemma:lsv_norm}
Let $M \in \R^{d \times d}$ be a positive definite matrix and $S \in \R^{d\times d}$ a positive definite symmetric matrix, 
then we have the following matrix norm inequality:
\begin{equation}
\norm{M S}^{2} \leq \sigma_{1}^{2}(M) \norm{S}^{2}.
\end{equation}
\end{lemma}

\begin{proof}
Using trace commutation property and relationship between the norm and the trace, we have
\begin{equation}\label{lemma:lsv:trace}
\norm{M S}^{2} = \Tr(S M^\trans M S) = \Tr(M^\trans M S^{2})
\end{equation}
Additionally, the J.~von Neumann's trace inequality \cite{Mirsky1975mfm} gives us
\begin{equation}\label{lemma:lsv:JVN}
\abs{\Tr(M_{1} M_{2})} \leq \sum_{i=1}^{d} \sigma_{i}(M_{1})\sigma_{i}(M_{2})  \leq \sigma_{1}(M_{1}) \sum_{i=1}^{d} \sigma_{i}(M_{2}),
\end{equation}
where $M_{1}$ and $M_{2}$ are any matrices in $\R^{d \times d}$.  
Combining \eqref{lemma:lsv:trace} and \eqref{lemma:lsv:JVN}, we get
\begin{equation*}
\norm{M S}^{2} \leq \sigma_{1}(M^\trans M) \sum_{i=1}^{d} \sigma_{i}(S^{2}) = \sigma_{1}^{2}(M) \norm{S}^{2}.
\end{equation*}
\end{proof}

\begin{remark}
Indeed, this lemma is quite obvious when looking at the variational definition of singular matrix, that is for a matrix $M$, the first or largest singular value is written as the solution of the following maximization program
$$
\sigma_{\max}(M) = \sup_{x,y \neq 0} \frac{x^T M y}{\|x\|_2\|y\|_2} = \sup_{y \neq 0} \frac{ \| M y\|_2}{\|y\|_2}
$$
Said differently, it also says that the maximum singular value is the $L_2$ operator norm of the matrix $M$.
\end{remark}

A powerful result that was initially proved in \cite{Akimoto_2012b} is that the covariance matrix of the NGD algorithm converges proportionally to the inverse of the Hessian matrix of our quadratic problem which is given by $\HMat$. To prove this convergence, we shows that the condition number of $\Sigma^{t} \HMat$ converges to 1 with a linear speed. This is the subject of the following lemma:
\begin{lemma}\label{lemma:cond_inverse_hessian}
If the learning rate $\nu_{\Sigma}^{t}$ for the covariance matrix satisfies inequalities \eqref{equation:learning_rate_cond_sigma}, then the condition number of $\Sigma^{t} \HMat$ converges to one with a linear rate of convergence given by
\begin{equation}
\limsup_{t\to\infty} \frac{\Cond(\Sigma^{t+1} \HMat) - 1}{\Cond(\Sigma^{t} \HMat) - 1} \leq \frac{1 - 2\nucmin}{1 - \nucmin}. \label{equation:rate-cond}
\end{equation}
Remarkably, we also have that the condition number is upper bounded by 
\begin{equation}
\Cond(\Sigma^{t} \HMat) \leq 1 + (1 - \nucmin)^{t}(\Cond(\Sigma^{0} \HMat) - 1). \label{equation:upper-cond}
\end{equation}
Finally, if the limit $\nuclim = \lim_{t\to\infty} \nu_{\Sigma}^{t} \lambda_{1}((\Sigma^{t})^{-1} \delta \Sigma^{t})$ exists, in equation \eqref{equation:rate-cond}, $\nucmin$ can be replaced by $\nuclim$ 
\end{lemma}

\begin{proof}
We are interested in the condition number of the matrix $\Sigma^{t} \HMat$. To make the form more symmetric, we can notice that it is better to look at the symmetrized term $\sqrt{\HMat} \Sigma^{t} \sqrt{\HMat}$. This term does exist because $\HMat$ is a positive definite and symmetric matrix, hence it admits a square root $\sqrt{\HMat}$. As $\sqrt{\HMat} \Sigma^{t} \sqrt{\HMat}$ is a positive definite and symmetric matrix, we can decompose it in an orthogonal matrix $O^{t}$ and a diagonal matrix $D_t$ as follows: $O^{t} D^{t} (O^t)^\trans$, where the diagonal elements of $D^{t} = \diag(D^{t}_{1}, \dots, D^{t}_{d})$ are the eigenvalues of $\sqrt{\HMat} \Sigma^{t} \sqrt{\HMat}$ and each column of $O^{t}$ is the eigenvector corresponding to each diagonal element of $D^{t}$. 
The orthogonal and diagonal matrix decomposition may not be unique but we will first prove that if we find an initial orthogonal matrix $O^0$ that diagonalizes $\sqrt{\HMat} \Sigma^{0} \sqrt{\HMat}$, it will also diagonalize the matrix $\sqrt{\HMat} \Sigma^{t} \sqrt{\HMat}$  for any $t \geq 0$.

Let us prove this result by mathematical induction
The result is true for step $t=0$ as $\sqrt{\HMat} \Sigma^{0} \sqrt{\HMat}$ is a positive definite and symmetric matrix, hence it admits an orthogonal and diagonal matrix decomposition  $O^{0} D^{0} (O^0)^\trans$, such that the diagonal elements of $D^{0} = \diag(D^{0}_{1}, \dots, D^{0}_{d})$ are the eigenvalues of $\sqrt{\HMat} \Sigma^{0} \sqrt{\HMat}$ and each column of $O^{0}$ is the eigenvector corresponding to each diagonal element of $D^{t}$. 
Let us assume that at the result holds at step $t$, that is $\sqrt{\HMat} \Sigma^{t} \sqrt{\HMat}$  admits an orthogonal and diagonal matrix decomposition such that
$$
\sqrt{\HMat} \Sigma^{t} \sqrt{\HMat} = O^{t} D^{t} (O^t)^\trans
$$
Using first lemma \ref{lemma:ng_algo}, we can use the covariance matrix update \eqref{equation:sigma-quad}. If we multiply both the right and the left side of the covariance matrix update \eqref{equation:sigma-quad} by $\sqrt{\HMat}$, we get:
\begin{equation*}
 \sqrt{\HMat} \Sigma^{t+1} \sqrt{\HMat} =\sqrt{\HMat} \Sigma^t \sqrt{\HMat} - \nu_\Sigma^{t} \,\kappa \, ( \sqrt{\HMat} \Sigma^t \sqrt{\HMat})^{2}.
\end{equation*}
Reformulating the above equation in terms of the orthogonal and diagonal matrix decomposition, we get:
\begin{equation*}
\sqrt{\HMat} \Sigma^{t+1} \sqrt{\HMat} = O^{t} \left(D^{t} - \nu_\Sigma^{t} \,\kappa \, (D^{t})^{2}\right) (O^{t})^\trans.
\end{equation*}
As $D^{t} - \nu_\Sigma^{t} \,\kappa \, (D^{t})^{2}$ is trivially a diagonal matrix, the above equation proves that the orthogonal $O^t$ also diagonalizes $\sqrt{\HMat} \Sigma^{t+1} \sqrt{\HMat}$. 
This proves that the result holds for step $t+1$. In addition, we get as a by product that the next step diagonal matrix $D^{t+1}$ is related to the previous step diagonal matrix $D^{t}$ as follows:
\begin{equation}\label{equation:lambda}
D^{t+1} = D^{t} - \nu_\Sigma^{t} \,\kappa \, (D^{t})^{2}.
\end{equation}
which is very interesting and is almost the update equation of our lemma \ref{lemma:condition_diagonal}. To get exactly the update equation, the trick is to introduce the diagonal matrix $\Lambda_{t}=c D^{t}$ as the constant $c>0$ and remark that equation \eqref{equation:lambda} is transformed into:
\begin{equation}
\Lambda^{t+1} =\Lambda^{t} - \nu_\Sigma^{t}  (\Lambda^{t})^{2}.
\end{equation}
which is exactly the update of the lemma \ref{lemma:condition_diagonal}. At this stage, we can notice various interesting remarks:
\begin{itemize}
\item the condition number of matrix $\Sigma^{t} \HMat$ is also equal to the one of the matrix $\sqrt{\HMat} \Sigma^{t} \sqrt{\HMat}$ itself equal to the condition number of $D^{t}$ itself equal to the condition number of $\Lambda^{t}$
\item the diagonal matrix series $\left( \Lambda_{t} \right)_{t\in \mathbb{N}}$ satisfies the assumption of lemma \ref{lemma:condition_diagonal}, hence we get
\begin{equation}
\frac{\Cond(\Lambda_{t+1}) - 1}{\Cond(\Lambda_{t}) - 1} \leq 1 - \numin
\end{equation}
or equivalently
\begin{equation}
\frac{\Cond(\Sigma^{t+1} \HMat) - 1}{\Cond(\Sigma^{t} \HMat) - 1} \leq 1 - \numin
\end{equation}
\end{itemize}

These remarks imply that the following upper bound for the condition number of $\Sigma^{t} \HMat$
$$
\Cond(\Sigma^{t} \HMat) \leq 1 + (1 - \nucmin)^{t}(\Cond(\Sigma^{0} \HMat) - 1)
$$ 
which exactly the equation \eqref{equation:upper-cond}. Moreover a trivial usage of squeeze (also known as the pinching or sandwich) theorem proves that
$$\lim_{t\to\infty}\Cond(\Sigma^{t} \HMat) = 1$$  
as the condition number of a matrix is always lower bounded by 1.
Using our first remark about the equality of the condition number of the matrix $\Sigma^{t} \HMat$ with the one of matrix $\Lambda_{t}$ and equation  \eqref{equation:cond-eq3}, we also get 
\begin{align*}
\limsup_{t \to \infty} \frac{\Cond(\Sigma^{t+1} \HMat) - 1}{\Cond(\Sigma^{t} \HMat) - 1} &= \limsup_{t\to\infty }\frac{1 - 2 \nu_{\Sigma}^{t}\lambda_{1}^{t}}{1 - \nu_{\Sigma}^{t}\lambda_{1}^{t}} \\
&\leq \frac{1 - 2\nucmin}{1 - \nucmin}.
\end{align*}
This proves the equation \eqref{equation:rate-cond}. 
Moreover, if the limit $\nuclim$ exists, from \eqref{equation:cond-im}, we see that one can replace $\nucmin$ by $\nuclim$ in \eqref{equation:rate-cond}. This concludes the proof.
\end{proof}

\begin{remark}
Lemma \ref{lemma:cond_inverse_hessian} is remarkable as it proves that the covariance gradually adapts to the inverse of the Hessian of our optimization problem, performing something quite similar to a Newton decrement.
\end{remark}

We can now prove the speed convergence part of proposition \ref{prop:convergence} as follows:
\begin{proof}
Using lemma \ref{lemma:ng_algo} and lemma \ref{lemma:lsv_norm}, we can apply the matrix and vector norm inequality $\norm{M x}^{2} \leq \sigma_{1}(M)^{2} \norm{x}^{2}$ to \eqref{equation:mu-quad} and \eqref{equation:mu-quad2}, we get 
\begin{equation*}
\frac{\norm{\mu^{t+1}}}{\norm{\mu^t}} \leq \sigma_{1}(I - \kappa \nu_{\mu}^{t} \Sigma^{t} \HMat).
\end{equation*}

Likewise, the same applies for the update of the covariance matrix $\Sigma$ given by equations \eqref{equation:sigma-quad} and \eqref{equation:sigma-quad2}, leading to a similar inequality for the matrix norm of the covariance
\begin{equation*}
\frac{\norm{\Sigma^{t+1}}}{\norm{\Sigma^t}} \leq \sigma_{1}(I - \kappa \nu_{\Sigma}^{t} \Sigma^{t} \HMat).
\end{equation*}

Let us look at the maximum singular value as follows:
\begin{equation*}
\begin{split}
&\limsup_{t} \sigma_{1}(I - \kappa \nu_\mu^{t} \Sigma^{t}\HMat)
\\
&= \limsup_{t}\sigma_{1}\biggl(I - {\nu_\mu^{t} \lambda_{1}(\kappa \Sigma^{t}\HMat)} {\frac{ \Sigma^{t}\HMat}{\lambda_{1}( \Sigma^{t}\HMat)}} \biggr) 
\end{split}
\end{equation*}
We can tackle these terms one by one. Using the lemma \ref{lemma:cond_inverse_hessian}, we have 
that 
\begin{equation*}
\lim_{t \to \infty} {\frac{ \Sigma^{t}\HMat}{\lambda_{1}( \Sigma^{t}\HMat)}}  = I_{d}
\end{equation*}
while assumptions \eqref{equation:learning_rate_cond_mu} states that
\begin{equation*}
\nu_\mu^{t} \lambda_{1}(\kappa \Sigma^{t}\HMat) \geq \nummin
\end{equation*}
leading to 
\begin{equation*}
\begin{split}
&\limsup_{t} \sigma_{1}(I - c \nu_\mu^{t} \Sigma^{t}\HMat) \\
&= \limsup_{t}\sigma_{1}\biggl(I - \nummin I_{d} \biggr) \leq 1 - \nummin.
\end{split}
\end{equation*}

This implies linear convergence of $\mu$ with rate of convergence at most $1 - \nummin$,
that is
\begin{equation}
\limsup \frac{\norm{\mu^{t+1}}}{\norm{\mu^t}} \leq 1 - \nummin, \label{equation:rate-k1}
\end{equation}

Likewise, the reasoning is almost the same for the covariance matrix, and we also get
\begin{equation*}
\begin{split}
&\limsup_{t} \sigma_{1}(I - \kappa \nu_\Sigma^{t} \Sigma^{t}\HMat) \leq 1 - \nucmin.
\end{split}
\end{equation*}
which proves the  linear convergence of $\Sigma$ with rate of convergence at most $1 - \nucmin$, that is
\begin{equation}
\limsup \frac{\norm{\Sigma^{t+1}}}{\norm{\Sigma^t}} \leq 1 - \nucmin, \label{equation:rate-k2}
\end{equation}
which ends the proof.
\end{proof}

\end{document}